%% file: main-colt.tex
\PassOptionsToPackage{dvipsnames}{xcolor}
\documentclass[final,12pt]{colt2025} %
\input{defs-kjun-v5-overleaf-colt}

\usepackage{microtype}
\usepackage{graphicx}
\usepackage{booktabs} %

\usepackage{amsmath}
\usepackage{amssymb}
\usepackage{mathtools}

\newcommand{\qedhere}{\ifmmode\hfill\ensuremath{\jmlrQED}\else\unskip\nobreak\hfill\ensuremath{\jmlrQED}\fi}

\usepackage{commath}

{\color{kjsaved}}%

\renewcommand{\cite}{\citep}

\usepackage{enumitem}
\setlist{nolistsep} %
\setlist{itemsep=.0em} %
\setlist[itemize]{topsep=.5pt,itemsep=0pt,parsep=2pt}
\setlist[enumerate]{topsep=.5pt,itemsep=0pt,parsep=2pt}

\newcommand{\jon}[1]{{\color{orange}[JR: #1]}}
\newcommand{\jk}[1]{{\color{red}[JK: #1]}}
\newcommand{\kj}[1]{{\color{ForestGreen}[KJ: #1]}}

\newif\ifFINAL
\FINALtrue %

\ifFINAL
  \def\blue#1{#1}
  \def\guide#1{}
  \def\gray#1{}
\renewcommand{\jon}[1]{}
\renewcommand{\jk}[1]{}
\renewcommand{\kj}[1]{}
\else
  \usepackage{transparent}
  \usepackage[inline]{showlabels}
  \usepackage{rotating}

\fi

\newcommand{\defeq}{\triangleq}
\newcommand{\E}{\EE}
\newcommand{\ie}{i.e., }
\newcommand{\wealth}{\mathsf{W}}
\newcommand{\wealthcrp}[1]{\mathsf{W}^{\mathsf{CRP}(#1)}}
\newcommand{\wealthup}{\mathsf{W}^{\mathsf{UP}}}
\newcommand{\wealthpcrp}{\mathsf{W}^{\mathsf{pCRP}^\star}}

\newcommand{\piref}{\pi_{\mathsf{ref}}}
\newcommand{\ips}{\tilr}

\newcommand{\mean}{\mu}
\newcommand{\var}{\mathsf{V}}
\usepackage{ifthen}
\newcommand{\lcbup}[2]{%
  \ifthenelse{\equal{#2}{}}%
    {\hat{\mean}_{\mathsf{UP}}}%
    {\hat{\mean}_{\mathsf{UP}}^{(#2)}}%
}
\newcommand{\lcbpcrp}[2]{%
  \ifthenelse{\equal{#2}{}}%
    {\hat{\mean}_{\mathsf{pCRP}^\star}}%
    {\hat{\mean}_{\mathsf{pCRP}^\star}^{(#2)}}%
}

\newcommand{\pcrplcb}{pCRP$^\star$-LCB}

\newcommand{\selectup}{\blue{{\hpi}_{\mathsf{UP}}^{(\dt)}}}
\newcommand{\selectupsimple}{\blue{{\hpi}_{\mathsf{UP}}}}

\newcommand{\selectpcrpsimple}{\blue{{\hpi}_{\mathsf{pCRP}^\star}}}

\newcommand{\lcbeb}[2]{\hat{\mean}_{\mathsf{EB}}^{(#2)}}

\usepackage{scrwfile}
\TOCclone[\appendixname]{toc}{atoc}
\newcommand\StartAppendixEntries{}
\AfterTOCHead[toc]{%
  \renewcommand\StartAppendixEntries{\value{tocdepth}=-10000\relax}%
}
\AfterTOCHead[atoc]{%
  \edef\maintocdepth{\the\value{tocdepth}}%
  \value{tocdepth}=-10000\relax%
  \renewcommand\StartAppendixEntries{\value{tocdepth}=\maintocdepth\relax}%
}

\usepackage{xparse}
\RenewDocumentEnvironment{proof}{ o }
  {%
    \par\noindent
    {\bfseries\upshape
      \IfNoValueTF{#1}
         {\proofname\ }
         {Proof of~#1\ }%
    }%
  }
  {%
    \jmlrQED
  }

\newtheorem{assumption}[theorem]{Assumption}

\numberwithin{theorem}{section}
\numberwithin{equation}{section}

\title[Improved Offline Contextual Bandits with Second-Order Bounds]
{Improved Offline Contextual Bandits with Second-Order Bounds:\\ Betting and Freezing}
\usepackage{times}
\coltauthor{%
 \Name{J. Jon Ryu} \Email{jongha@mit.edu}\\
 \addr Massachusetts Institute of Technology
 \AND
 \Name{Jeongyeol Kwon} \Email{jeongyeol.kwon@wisc.edu}\\
 \addr University of Wisconsin-Madison%
 \AND
 \Name{Benjamin Koppe} \Email{bek76@cornell.edu}\\
 \addr Cornell University%
 \AND
 \Name{Kwang-Sung Jun} \Email{kjun@cs.arizona.edu}\\
 \addr University of Arizona%
 \AND%
}

\begin{document}

\maketitle

\begin{abstract}%
We consider off-policy selection and learning in contextual bandits, where the learner aims to select or train a reward-maximizing policy using data collected by a fixed behavior policy. Our contribution is two-fold. First, we propose a novel off-policy selection method that leverages a new betting-based confidence bound applied to an inverse propensity weight sequence. Our theoretical analysis reveals that this method achieves a significantly improved, variance-adaptive guarantee over prior work. Second, we propose a novel and generic condition on the optimization objective for off-policy learning that strikes a different balance between bias and variance. One special case, which we call freezing, tends to induce low variance, which is preferred in small-data regimes. Our analysis shows that it matches the best existing guarantees. In our empirical study, our selection method outperforms existing methods, and freezing exhibits improved performance in small-sample regimes.
\end{abstract}


\textfloatsep=1em

\section{Introduction}

The offline contextual bandit problem has emerged as a critical area of study in sequential decision-making, with significant implications for decision systems for various domains including recommendation~\citep{li10acontextual} and online advertising~\citep{schwartz2017customer}.
In this problem, a behavior policy $\piref(a|x)$ is deployed in the environment for a nontrivial period of time, where the policy defines a conditional distribution over the actions $a$ (e.g., items to be recommended) given each context information $x$ (e.g., user being served).
Specifically, at each time step $t\in[n] \defeq \{1,\ldots,n\}$, an agent observes a context $x_t\sim \cD$ from an unknown distribution $\cD$, takes an action $a_t \sim \piref(a|x_t)$, and then receives a reward $r_t = r(x_t, a_t) \in [0,1]$ where $r$ is an unknown  (possibly stochastic) reward function. 
Given offline logs of interactions $D_n\defeq \{(x_t,a_t,r_t)\}_{t=1}^n$ obtained via the behavior policy, we wish to find a policy $\pi$ that maximizes the expected reward $\mean(\pi)\defeq \EE_{x\sim\cD,a\sim \pi(a|x)}[r(x,a)]$, which we call the \emph{value} of the policy. 
This setting is called \emph{off-policy}, in contrast to its online counterpart, where a policy can be updated continually using feedback from the environment.
While online interactions may allow for more effective policy optimization, in many real-world scenarios this is either infeasible due to system constraints or too costly due to operational risks.
The offline problem naturally arises as a viable alternative in this context.

\begin{figure*}
  \begin{center}
  \scalebox{.975}{
    \begin{tabular}{c}
      $\underbrace{\sqrt{\EE\biggl[\fr{(\ips_1^{\pi^*} - v(\pi^*))^2}{1 + \beta 
          (\ips_1^{\pi^*}\!-\!v(\pi^*))}\biggr]}}_{\tsty\text{PUB (\textbf{ours}; see Eq.~\eqref{eq:selection_refined})}} 
      \le \displaystyle\EE\biggl[\fr{(\ips_1^{\pi^*} - v(\pi^*))^2}{1 + \beta 
          (\ips_1^{\pi^*}\!-\!v(\pi^*))}\biggr]
      \lesssim\!\! \underbrace{1 \!+\! \EE\biggl[\fr{(\ips_1^{\pi^*})^2}{1 + \beta 
          \ips_1^{\pi^*}}\biggr]}_{\tsty\text{\citet{sakhi24logarithmic}} } 
      \le\!\! \underbrace{1 \!+\! \EE\biggl[
        \frac{(\ips_1^{\pi^\star})^2}{\pi^\star r+\beta\ips_1^{\pi^\star}}
        \biggr]}_{\tsty\text{\citet{gabbianelli24importance}} }$.
\vspace{.5em}\\
(a) \textbf{Off-Policy Selection} 
\vspace{1em}\\
       $  
        \underbrace{\EE\biggl[
        \frac{(\ips_1^{\pi^\star})^2}{c_u + \gamma\ips_1^{\pi^\star}}
        \biggr]}_{ 
             \tsty\shortstack{\text{Solve \eqref{eq:op-learning-alg-0} + Assumption~\ref{ass:phi-new} (\textbf{ours}; see Eq.~\eqref{eq:learning})} \\ \text{\& \citet{sakhi24logarithmic}}} 
         } 
       \le \underbrace{
       \EE\biggl[
        \frac{(\ips_1^{\pi^\star})^2}{\pi^\star r+\gamma\ips_1^{\pi^\star}}
        \biggr]
        }_{\tsty\text{\citet{gabbianelli24importance}}}. 
       $
       \vspace{.5em}\\
       (b) \textbf{Off-Policy Learning}
    \end{tabular}}
  \end{center}\vspace{-1em}
  \caption{
  Comparison of different bounds on the offline regret (see Eq.~\eqref{eq:offline_regret}) for the off-policy (a) selection and (b) learning where $\beta \approx \sqrt{1/n}$ in (a), and $\gamma>0$ in (b) is a hyperparameter.
  We hide a factor of $\sqrt{1/n}$ and other constants. 
  The symbol above $\lesssim$ holds for $n$ sufficiently large. %
  For selection, our method achieves an improved bound.
  For learning, we propose a broad family of methods that achieves the same order of bound as~\citet{sakhi24logarithmic}.
  }
  \label{fig:comparison}
\end{figure*}

The main challenge in such an offline setting is in the discrepancy between the behavior policy used to log the offline data and the set of candidate policies whose performance we wish to evaluate.
That is, we cannot simply use the offline data $D_n$ to estimate the expected reward of an arbitrary policy $\pi$, since $\pi$ may choose actions that are different from $a_t$'s chosen by $\piref$, in which case we have not observed the corresponding rewards.
This is in stark contrast to the supervised learning setup where a classifier's generalization error can be estimated by simply computing the average error on a test dataset.
To circumvent the problem, researchers have proposed unbiased estimators of the expected reward of a policy, such as the Importance Weighted (IW) estimator\footnote{
Also known as Inverse Propensity Score (IPS) or Inverse Propensity Weighting (IPW) estimators.
}~\citep{horvitz52generalization,liu19competing} 
and Doubly Robust (DR) estimator~\citep{robins95semiparametric}, with numerous extensions of them.
The IW estimator is defined as
\begin{align}\label{eq:iw}
    \hat{\mean}_n^{\mathsf{IW}}(\pi) 
    \defeq 
    \frac{1}{n}\sum_{t=1}^n \ips_t^\pi, \text{~~ where ~~} \ips_t^\pi = w_t^\pi r_t
    ,
\end{align}
is called the \emph{importance-weighted reward}, and we refer to $w_t^\pi\defeq \frac{\pi(a_t|x_t)}{\piref(a_t|x_t)}$ as the importance weight.

Depending on the goal, there are three representative types of off-policy (OP) problems.
Below, we contrast each with its counterpart in a supervised learning setup.

\begin{itemize}
\item \textbf{Off-policy evaluation}:
Given a policy $\pi$, we wish to estimate its value (i.e., expected reward).
In supervised learning, this corresponds to estimating the generalization error of a classifier using test data or obtaining confidence bounds for it.

\item \textbf{Off-policy selection}:
Given $\Pi$, a finite set of candidate policies, we wish to find the best policy---that is, the one with the highest expected reward (i.e., value).
In supervised learning, this corresponds to performing model selection using hold-out validation data.

\item \textbf{Off-policy learning (optimization)}:
Given a policy class $\Pi$ (typically with $|\Pi| = \infty$, as in the case of neural-network policies), we wish to find the best policy $\pi$ that achieves the highest value.
In supervised learning, this corresponds to learning a classifier using training data.
\end{itemize}

In selection and learning, we wish to establish a guarantee on the \emph{offline regret} (or suboptimality gap) for the policy $\hat{\pi} \in \Pi$ selected by an algorithm, which is defined as, for $\pi^\star \defeq  \arg\max_{\pi\in\Pi} \mean(\pi)$,
\begin{align}
\label{eq:offline_regret}
    \mathsf{Reg}_n(\hat{\pi}) 
    \defeq \mean(\pi^\star)-\mean(\hat{\pi}).
\end{align}

We note that the key difference between selection and learning lies in the cardinality of the policy class $\Pi$. 
In the selection problem, since $\Pi$ is finite, we can exhaustively evaluate the performance of each policy.
In the learning (optimization) problem, however, $\Pi$ is typically a continuously parameterized class (e.g., neural networks), and thus solving it requires \emph{computational efficiency} in the optimization process.
In the literature, such requirement on the learning objective is called \emph{oracle efficiency}~\citep{langford07epoch,wang24oracle}, meaning that the learning objective is optimizable efficiently assuming access to an optimization oracle.
More concretely, we prefer objectives that are convex, or at least amenable to stochastic gradient-based optimization.
\paragraph{Contributions.}
In this paper, we make two main contributions.
First, we propose a novel OP selection method called PUB (\textbf{P}essimism via semi-\textbf{U}nbounded-coin-\textbf{B}etting).
PUB is an algorithm for computing a lower confidence bound (LCB) of any nonnegative random variable and is based on a variation of betting-based confidence bound~\citep{Waudby-Smith--Ramdas2020b,orabona24tight,Ryu--Bhatt2024}.
By applying our new LCB to the importance-weighted rewards $\{\til r_t^\pi\}_{t=1}^n$ (defined in Eq.~\eqref{eq:iw}) for each policy under consideration, we can establish a guarantee on the performance measure called offline regret (defined in Eq.~\eqref{eq:offline_regret}), which is strictly tighter than existing works to our knowledge.
We highlight two features in our guarantee. First, our regret bound scales with the \emph{standard deviation} of $\ips^\pi$, significantly improving the prior art scaling with the raw second moment. Second, more crucially, we achieve the improved guarantee \emph{without any hyperparameter tuning}. This is crucial in practice as tuning a parameter in the existing estimators is infeasible in general due to the lack of knowledge on the second-order statistics of $\ips^\pi$.
We summarize the comparison in Figure~\ref{fig:comparison}(a) and provide details in Section~\ref{sec:op-selection}.
Our LCB can be also applied to OP evaluation to construct both lower- and upper- confidence bounds, provably converging to the value of a target policy; we defer this discussion to Appendix~\ref{sec:evaluation}. 
Second, we propose a broad family of optimization objectives for OP learning in the form of
\begin{align}\label{eq:op-learning-alg-0}
    \hpi_n\defeq \arg\max_{\pi \in \Pi} 
    \sum_{t=1}^n \phi (\beta \ips_t^\pi),
\end{align}
where $\beta$ is a hyperparameter, and $\phi\colon \mathbb{R}_+\to\mathbb{R}$ is a \emph{score function}.
Under a mild condition on $\phi$, we show that such an optimal policy $\hpi_n$ guarantees a regret bound that depends on a second moment, matching the rate of the state-of-the-art method known as \emph{logarithmic smoothing}~\citep{sakhi24logarithmic}.
One extreme instance of our generic family is called \textit{freezing}, in which the score function $\phi(x)$ is zero for $x$ sufficiently large.
This greatly reduces variance at the cost of introducing bias, and we empirically show that freezing achieves the best performance especially in the small-data regime.
Our analysis not only matches the same smoothed second-order bound as the state-of-the-art~\citep{sakhi24logarithmic}, but also reveals that, depending on the problem instance, more aggressive methods such as freezing may be preferable.
We summarize our achieved bounds along with those of existing work in Figure~\ref{fig:comparison}(b), and explain the details in Section~\ref{sec:learning}.

Finally, we conduct an empirical evaluation of the proposed selection and learning methods, following the suite of experiments in \citep{wang24oracle}.
We demonstrate that PUB outperforms all existing methods, and that the new learning methods either outperform or match the performance of baseline methods.
We conclude the paper by outlining promising directions for future research.
Due to space constraints, we defer the discussion of related work to Appendix~\ref{sec:related}.

\paragraph{Notation.} For a random variable $X$, we denote its expectation and variance by $\EE[X]$ and $\VV[X]=\EE[X^2]-\EE[X]^2$, respectively.
We use $a_{1:n}$ to denote a sequence of numbers $a_1,\ldots,a_n$.
For real numbers $a,b\in\mathbb{R}$, we use shorthand notations $a\wed b\defeq \max\{a,b\}$ and $a\vee b \defeq \min\{a,b\}$.

\section{Problem Setting}
We are given a log of interactions $D_n=\{(x_t,a_t,r_t)\}_{t=1}^n$ from a contextual bandit, collected using a behavior (or reference) policy $\piref(a|x)$.
That is, for each $t\ge 1$, 
$
(x_t,a_t,r_t)\sim p(x) \piref(a|x) p(r|x,a).
$
Based on the bandit-logged data $D_n$, our goal is to evaluate the value of a target policy $\pi(a|x)$:
\[
\mean(\pi)\defeq \E_{(x,a,r)\sim p(x)\pi(a|x)p(r|a,x)}[r] ~.
\]
With a slight abuse of notation, we will occasionally write $r=r(x,a)$, where $r(x,a) \in[0,1]$ denotes a (possibly stochastic) reward function.
One simple yet popular unbiased estimator for $\mean(\pi)$ is the importance weighted (IW) estimator~\citep{horvitz52generalization} defined in Eq.~\eqref{eq:iw}.
Hereafter, we denote the variance of the importance-weighted reward by
\[\tilde{\var}(\pi)\defeq \VV[\ips_1^\pi].\]
While the IW estimator is unbiased, i.e., $\EE[\hat{\mean}_n^{\mathsf{IW}}(\pi) ]=\mean(\pi)$, the variance $\tilde{\var}(\pi)$ can be undesirably large for policies that frequently choose different actions not explored by $\piref$.
The effect is exacerbated when the IW estimator is used as the objective function for selection and/or learning tasks.
That is, we may end up choosing a poor policy because, if $\ips_1^\pi$ exhibits disproportionately high variance, then the value can be largely overestimated with nontrivial probability.

This led to development of the \textit{pessimism} principle~\citep{swaminathan15batch}, which aims to find the policy that maximizes a lower confidence bound on the value.
This has the benefit of penalizing policies with large variance, effectively serving as a form of regularization to promote stability.
Theoretically, pessimism is known to enjoy a property called \textit{single-policy concentrability}, which means that the primary factor determining the convergence of the offline regret (see Eq.~\eqref{eq:offline_regret}) scales with a quantity that depends on the variability of the optimal policy $\pi^*$, rather than the worst-case variability over all policies $\pi\in\Pi$, a condition commonly referred to as \textit{all-policy concentrability}.
Intuitively, single-policy concentrability ensures that the convergence to the optimal policy is not affected by ill-behaved policies, as long as the optimal policy remains well-behaved.

\section{Off-Policy Selection}
\label{sec:op-selection}
In the selection problem, we wish to choose a policy that maximizes the expected reward from a set of finite policies.
Developing a new lower confidence bound technique,
we will follow the standard \emph{pessimism} under uncertainty: construct the LCB on the expected reward for each policy, and choose the policy that maximizes the LCB.
In what follows, we first introduce a betting-based (time-uniform) confidence bound for mean-parameter estimation when the random variables are $[0,\infty)$-valued.
We then show how the pessimism strategy with our confidence bound performs in the selection task (Theorem~\ref{thm:main_selection}) and discuss its superiority against existing methods.

\subsection{New Betting-Based Lower Confidence Bound for \texorpdfstring{$[0,\infty)$}{[0,infty]}-Valued Random Variables}
To construct a confidence bound, 
we draw ideas from gambling and the martingale theory, which have been widely used in the recent literature~\citep{orabona24tight,Waudby-Smith--Ramdas2020a,Waudby-Smith--Ramdas2020b,Waudby-Smith--Wu--Ramdas--Karampatziakis--Mineiro2022,Ryu--Bhatt2024,Ryu--Wornell2024}.
The most general form of gambling is stock market investment~\citep{Cover--Thomas2006}, but we focus on betting in a two-stock market, following the convention of \citet{Ryu--Bhatt2024}.

\subsubsection{A Generic Betting-Based Construction}
Suppose that there are two stocks, say stock 1 and stock 2. 
On each day $t\in\mathbb{N}$, a gambler must make her betting $\bb_t=(b_t,1-b_t)$, for $b_t\in[0,1]$, over the two stocks at the beginning.
At the end of the day, the \emph{price relative vector} $\bx_t=(x_{t1},x_{t2})\in\mathbb{R}_+^2$ is revealed, where $x_{ti}>0$ captures the multiplicative change in the price of stock $i$.
Note that the betting $\bb_t$ must be causal, that is, $\bb_t$ can be only a function of the past observations $\bx_{1:t-1}$.
If we denote the gambler's wealth at day $t$ by $\wealth_t$, then the multiplicative gain of the wealth can be written as 
\[
\frac{\wealth_t}{\wealth
_{t-1}}=\bb_t^\intercal\bx_t
=b_t x_{t1} + (1-b_t) x_{t2}.
\]
If we assume that $(\bx_t)_{t=1}^\infty$ is stochastic and satisfies $\EE[\bx_t|\bx_{1:t-1}]\le [1,1]^\intercal$ (coordinate-wise), 
then it is easy to check that the wealth process $(\wealth_t)_{t=1}^\infty$ is \emph{super-martingale}, \ie $\EE[\wealth_t|\bx_{1:t-1}]\le \wealth_{t-1}$, regardless of the choice of $\bb_t$.

Now, suppose that we have a random process $(Y_t)_{t=1}^{\infty}$ such that $\EE[Y_t|Y^{t-1}] = \EE[Y_1] \defeq \blue{\mu}$ for any $t\ge1$, and we wish to construct a confidence set for the unknown mean parameter $\mu$.
We can then construct a confidence sequence based on betting as follows.
First, we construct a \emph{hypothetical} stock market $\bx_t(\nu)$ as a function of candidate mean parameter $\nu$, such that any wealth process from the market becomes super-martingale when $\nu=\mu$.
If we denote the resulting wealth by $\wealth_n(\bx_{1:n}(\nu))$,
then, by applying Ville's inequality~\citep{ville39etude} to a super-martingale $(\wealth_t(\bx_{1:t}(\mu)))_{t=1}^{\infty}$, we have
\begin{align*}
  1-\dt \le \PP\del[3]{ \sup_{n\ge 1} \wealth_n(\bx_{1:n}(\mu)) < \fr1\dt }.
\end{align*}
Given that this good event happens, we can now construct a confidence set at level $(1-\dt)$ at each time step $t$, by collecting all candidate parameters that result in wealth that \emph{does not exceed} the threshold $1/\delta$, as they cannot be $\mu$.
This outlines the general recipe for constructing confidence sequences based on betting.
To derive a confidence sequence based on this meta-algorithm, one needs to specify: (1) how to construct the hypothetical stock market, and (2) which betting strategy to employ.\looseness=-1

The construction for bounded random processes have been extensively studied in the recent literature~\citep{Waudby-Smith--Ramdas2020a,Waudby-Smith--Ramdas2020b,orabona24tight,Ryu--Bhatt2024}.
Specifically, for a $[0,1]$-valued random process $(Y_t)_{t=1}^\infty$ with mean parameter $\mu\in(0,1)$, one can set the stock market $\bx_{t}(\nu) \defeq [\fr{Y_t}{\nu}, \fr{1-Y_t}{1-\nu}]^\intercal$ for each $\nu\in(0,1)$.
In particular, \citet{orabona24tight} showed that a variant of Cover's universal portfolio leads to confidence bounds that are of empirical-Bernstein type (i.e., confidence width adapts to the empirical variance), and provably never worse than the Bernoulli-KL-based confidence bound.
The latter property does not hold for the empirical Bernstein bound~\cite{maurer09empirical} in the small-sample regime.

What if we are interested in a nonnegative random process \((Y_t)_{t=1}^{\infty}\) that may be unbounded (i.e., \(Y_t \in [0,\infty)\))?
The unbounded nature of \(Y_t\) breaks the nonnegativity of the market sequence \(\bx_t(\nu)\) in the construction above, thereby violating Ville's inequality and preventing us from obtaining confidence bounds.  
As a solution, \citet{Waudby-Smith--Wu--Ramdas--Karampatziakis--Mineiro2022} considered a stock market which, when rephrased in the two-stock market language of \citet{Ryu--Bhatt2024}, takes the form
\begin{align*}
\bx_t(\nu)\defeq \biggl[\frac{Y_t}{\nu},1\biggr]^\intercal,
\end{align*}
so that the resulting wealth process remains nonnegative.
Here, the first stock depends on the underlying process $Y_t$, while the second stock can be interpreted as \emph{cash}.
We call this a \emph{one-sided betting} formulation.
For a betting strategy $(\bx_{1:t-1}(\nu)\mapsto \bb_t)_{t=1}^\infty$, the cumulative wealth is then
\begin{align*}
\wealth_n(\bx_{1:n}(\nu))
\defeq \wealth_n(Y_{1:n}; \nu)
\defeq \prod_{t=1}^n \biggl(1-b_t+b_t\frac{Y_t}{\nu}\biggr),
\end{align*}
assuming that we start from a unit initial wealth $\wealth_0=1$.
Thus, for $\dt\in(0,1)$, defining
\begin{align*}
C_{\mathsf{bet}}^{(\dt)}(Y_{1:n})
\defeq \biggl\{
\nu\in(0,1)\colon \wealth_n(Y_{1:n};\nu)\le \frac{1}{\dt}
\biggr\}
\quad\text{and}\quad
\hat{\mean}_{\mathsf{bet}}^{(\dt)}(Y_{1:n})\defeq \inf C_{\mathsf{bet}}^{(\dt)}(Y_{1:n}),
\end{align*}
we have that
$\hat{\mean}_{\mathsf{bet}}^{(\dt)}(Y_{1:n})$ is a $(1-\dt)$-lower confidence bound (LCB) for $\EE[Y_1]$ by Ville's inequality:
\begin{proposition}[{\citealp[Proposition~1]{Waudby-Smith--Wu--Ramdas--Karampatziakis--Mineiro2022}}]
\label{prop:lcb}
{Let $(Y_t)_{t=1}^\infty$ a non-negative real-valued random process $(Y_t)_{t=1}^\infty$ such that $\EE[Y_t|Y^{t-1}]\defeq \EE[Y_1]=\mu>0$ for any $t\ge 1$.}
For any causal betting strategy, $C_{\mathsf{bet}}^{(\dt)}(Y_{1:n})$ is a (time-uniform) lower confidence set at level $1-\dt$, that is,
\begin{align*}
\PP\Bigl(\forall t \ge 1, \mu \ge \hat{\mean}_{\mathsf{bet}}^{(\dt)}(Y_{1:t})\Bigr) \ge 1-\dt.
\end{align*}    
\end{proposition}
The proof of the result above, deferred to Appendix, is based on a standard martingale argument.
We note that $(\hat{\mean}_{\mathsf{bet}}^{(\dt)}(Y_{1:t}))_{t=1}^\infty$ satisfies a strong, \emph{time-uniform} guarantee, but in its application to OP problems below, we will only invoke the LCB only for the last time step.
We also note in passing that obtaining an \textit{upper} confidence bound for $[0,\infty)$-valued random variables is nontrivial and is beyond the scope of our work.
While any choice of betting strategy yields a valid, time-uniform LCB, we need a \emph{good} betting strategy to obtain a \textit{tight} (i.e., sample-efficient) LCB.
Specifically, since the LCB is a random variable as well, one may wish to establish its sample efficiency of the LCB by bounding the gap between the LCB and the true mean.
To make such a bound meaningful, it is important that the bound consists of \textit{deterministic} quantities (e.g., variance) because random quantities (e.g., the empirical variance) may behave poorly---even with a large number of samples---rendering the bound unreliable.
This consideration is particularly important for $[0,\infty)$-valued random variables, for which the empirical variance may fail to converge.\footnote{
  For $[0,1]$-valued random variables, one can easily show that the empirical variance will converge to the true variance.
}
To the best of our knowledge, however, existing LCBs for this setting either do not provide any sample efficiency guarantees~\citep{Waudby-Smith--Wu--Ramdas--Karampatziakis--Mineiro2022}, or provide regret guarantees that scale worse than $\sqrt{\Var(\ips^\pi_1)}$ (or a comparable quantity)~\citep{gabbianelli24importance,sakhi24logarithmic}, which we consider to be a desirable dependence. 
Alternatively, some methods rely on additional assumptions, such as an upper bound on the variance or convergence of the sample variance~\citep{wang2023catoni}.
This observation motivates the novel LCB we introduce below.

\subsubsection{LCB Induced by Universal Portfolio}
We propose to use \citet{Cover1991}'s universal portfolio as the betting strategy, in the same spirit as \citet{orabona24tight}, who first studied its application to \emph{bounded} processes.
A constant betting strategy $\bb_t=(b,1-b)$, also referred to as a \emph{constantly rebalanced portfolio} (CRP), yields
\begin{align}
\label{eq:wealth_crp}
\wealthcrp{b}_n(Y_{1:n};\nu)
&\defeq \prod_{t=1}^n \del[2]{1-b+b\frac{Y_t}{\nu}}
\end{align}
as the cumulative wealth, for some $b\in[0,1]$.
\citet{Cover1991} proposed a strategy called the \emph{$w$-weighted universal portfolio}, or \emph{universal portfolio} (UP) in short,
to track the wealth achieved by the best CRP in hindsight asymptotically up to the first-order exponent.
Cover's UP is defined as the \emph{mixture} of CRP wealths with respect to a mixture distribution $w(b)$ over $b\in[0,1]$, that is,
\begin{align}
\label{eq:wealth_up}
\wealthup_n(Y_{1:n};\nu)
&\defeq \int_0^1 \wealthcrp{b}_n(Y_{1:n};\nu) w(b) db.
\end{align}
Intuitively, Cover's UP can be understood as a buy-and-hold strategy of the set of constant betting strategies, where a unit wealth is distributed according to the weight $w(b)$.
\citet{Cover--Ordentlich1996} showed that, with the particular choice of weight distribution $w(b)=\frac{1}{\sqrt{\pi b(1-b)}}$, which is the density of the $\mathsf{Beta}(\frac12,\frac12)$ distribution and the default in our work,
the UP's wealth is minimax optimal with respect to the class of CRPs.
For our specific stock market, the guarantee of \citet[Theorem~2]{Cover--Ordentlich1996} simplifies as follows: for any sequence $y_{1:n}\in\mathbb{R}_+^n$,
\begin{align}
{\wealthup_n(y_{1:n};\nu)}
\ge 
\wealthpcrp_n(y_{1:n};\nu) 
\defeq \frac{1}{\sqrt{\pi(n+1)}}\sup_{b\in(0,1)}\wealthcrp{b}(y_{1:n};\nu).
\label{eq:up_vs_pcrp}
\end{align}
In words, this shows that the Cover's UP can achieve the performance of the best CRP up to a polynomial factor $\sqrt{\pi(n+1)}$.
We refer to the right-hand side as the \emph{penalized best CRP wealth}.

In Figure~\ref{fig:ex_betting_lcb}, we visualize the logarithmic wealth functions $\nu\mapsto\ln\wealthup(Y_{1:n};\nu)$ of different CRPs and that of Cover's UP, at different time steps.
We first note that the wealth function of each CRP is log-convex and monotonically decreasing, and thus so is that of Cover's UP. 
This ensures that there exists a unique root for the equation $\wealthup(Y_{1:n};\nu)=\dt^{-1}$, and thus
\begin{align}\label{eq:uplcb}
    \lcbup{n}{\dt}(Y_{1:n})    \defeq \min\cbr[2]{\nu>0: \wealthup_n(Y_{1:n};\nu) \le \fr1\dt}
\end{align}
is well-defined and a valid $(1-\dt)$-LCB.
We refer to the resulting bound as the \emph{UP-LCB}.
We also remark that the curve of Cover's UP, $\nu\mapsto\wealthup_n(Y_{1:n};\nu)$, closely approximates the frontier $\nu\mapsto\sup_{b\in[0,1]}\wealthcrp{b}_n(Y_{1:n};\nu)$.
This follows from the fact that Cover's UP asymptotically tracks the wealth of the best CRP for any stock market, as implied by Eq.~\eqref{eq:up_vs_pcrp}.

One can also use the penalized best CRP wealth in Eq.~\eqref{eq:up_vs_pcrp} to construct an LCB, which is slightly looser than UP-LCB, yet simpler to compute; see below for computational details.
Specifically, 
\begin{align}
    \lcbpcrp{n}{\dt}(Y_{1:n}) \defeq \min\cbr[2]{\nu>0: \wealthpcrp(Y_{1:n};\nu) \le \fr{1}{\dt}}
\label{eq:pcrplcb}
\end{align}
is a valid $(1-\dt)$-LCB,
which we call the \emph{penalized-best-CRP-LCB} or \emph{\pcrplcb{}} in short.

{
\subsubsection{Finite-Sample Guarantees}
\label{sec:finite_sample}
Our main technical contribution in this section is the following statement, which establishes the rate of convergence of UP-LCB and \pcrplcb{} to the true mean, automatically adapting to the underlying variance. 
For $n$ sufficiently large, we further show that the convergence is proportional to a \emph{smoothed variance}, defined as follows. This guarantee will be handy later for comparing the bound with \citep{sakhi24logarithmic}.
For a nonnegative random variable $Y$, we define a \emph{$b$-smoothed variance}
\begin{align*}
\WW_b[Y]\defeq \EE\biggl[\fr{(Y-\EE[Y])^2}{1+b\fr{Y-\EE[Y]}{\EE[Y]}}\biggr]
= \EE[Y]
\EE\biggl[\fr{(Y-\EE[Y])^2}{bY+(1-b)\EE[Y]}\biggr]
\end{align*}
for $b\in[0,1]$.
We note that $\WW_b[Y]$ interpolates the two extreme quantities $\WW_0[Y]=\VV[Y]$, the variance, and $\WW_1[Y] = \EE[Y]\EE[\fr{(Y-\EE[Y])^2}{Y}]$.
Under a mild assumption, i.e., $\EE[Y]<\infty$ and $\EE[\fr{1}{Y}]<\infty$,
$b\mapsto \WW_b[Y]$ is strictly convex, unless $Y$ is constant with probability 1.
Under such condition, 
\begin{align*}
\WW_b[Y] < \WW_0[Y]\wed \WW_1[Y] = \VV[Y] \wed \EE[Y]\EE\biggl[\fr{(Y-\EE[Y])^2}{Y}\biggr].
\end{align*}

In what follows, we assume that $(Y_t)_{t=1}^\infty$ is an independent and identically distributed (i.i.d.), nonnegative random process, with 
\[
\text{$\mu\defeq\EE[Y_1]$ 
\quad and \quad 
$\sigma^2\defeq \VV[Y_1].$}
\]
\begin{theorem}[Convergence rate of UP-LCB and \pcrplcb{}]
\label{thm:up_lcb_rate}
Let $n\ge 1$ and
define
$\blue{F_n^{(\dt)}} \defeq \ln\fr{\sqrt{\pi(n+1)}}{\dt^2}$.
Then, with probability $\ge 1-2\dt$,
\[
0
\le \mu - \lcbup{n}{\dt}(Y_{1:n})
\le \mu - \lcbpcrp{n}{\dt}(Y_{1:n})
\le \sqrt{\fr{48\sigma^2}{n} F_n^{(\dt)}} \vee \fr{12\mu}{n}F_n^{(\dt)}.
\]
Moreover, if 
$n\ge 108\bigl(1\vee 36\fr{\mu^2}{\sigma^2}\bigr) F_n^{(\dt)}$, for $b_n^{(\dt)}\defeq \sqrt{\fr{\mu^2}{2\sigma^2}\fr{F_n^{(\dt)}}{n}}$,
we further have
\begin{align}
\mu - \lcbpcrp{n}{\dt}(Y_{1:n})
&\le 
\inf_{b\in(0,\fr34]}
\biggl\{
\frac{b}{\mu}
\WW_b[Y_1]
+\fr{\mu}{b}
\fr{F_n^{(\dt)}}{n}
\biggr\}
\le 2\sqrt{
\frac{\WW_{b_n^{(\dt)}}[Y_1]}{n}
F_n^{(\dt)}}.
\label{eq:lcb_rate_refined}
\end{align}
\end{theorem}
This shows that both UP-LCB and \pcrplcb{} converge to the true mean from below at the rate of $O(\sqrt{\frac{\sigma^2}{n}\ln\frac{n^{1/4}}{\dt}}\vee \frac{\mu}{n}\ln\frac{n^{1/4}}{\dt})$.
It is analogous to Bernstein's inequality, but importantly, it holds uniformly over time and applies to any $[0,\infty)$-valued random variables.
To our knowledge, this is the first LCB with a finite-sample guarantee for $[0,\infty)$-valued random variables, which is much stronger than merely statistically valid bounds. 
}

\begin{figure}[t]
    \centering
    \vspace{-.25em}    \includegraphics[width=.8\linewidth]{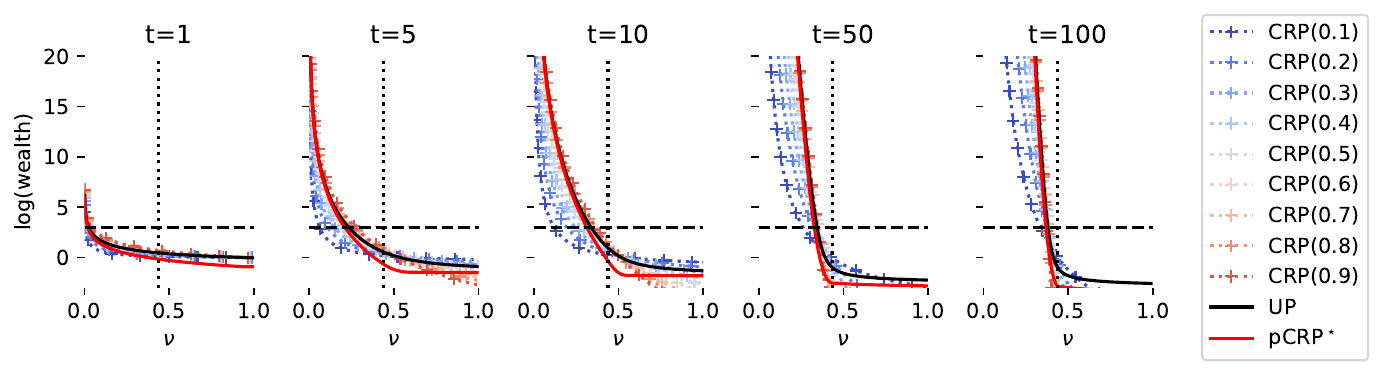}
    \vspace{-1em}
    \caption{Example of the evolution of cumulative wealths achieved by different CRPs in Eq.~\eqref{eq:wealth_crp}, Cover's UP in Eq.~\eqref{eq:wealth_up}, and the penalized best CRP wealth in Eq.~\eqref{eq:up_vs_pcrp}.
    The underlying process is a sequence of independent and identically distributed (i.i.d.) Gamma random variables with shape and scale parameters of 6 and $1/8$, respectively, and thus mean $3/4$.
    }\label{fig:ex_betting_lcb}\vspace{-1em}
\end{figure}

\paragraph{Empirical-Bernstein-type relaxation.}
We can also derive a loose outer bound of the UP-LCB and \pcrplcb{}, which is simply a function of the empirical mean and variance in a similar spirit to \citep[Theorem~6]{orabona24tight}; see Theorem~\ref{thm:emp_bern_relaxation} for a formal statement.
This relaxation can be understood as a statistically valid empirical Bernstein inequality for $[0,\infty)$-valued random variables.
While this bound is statistically valid, it does not characterize the rate of convergence or even the asymptotic consistency of the LCB as an estimator of the mean.

\vspace{.25em}\noindent\textbf{Implementation.~}
We can compute the UP-LCB up to numerical precision using a dynamic programming approach combined with binary search over $\nu$, similar to the method described in \citep{Ryu--Bhatt2024}.
Its time complexity for processing a length-$n$ trajectory is, however, $O(n^2)$.
In practice, one can implement \pcrplcb{} (Eq.~\eqref{eq:pcrplcb}), or techniques in \citet{orabona24tight} or \citet{Ryu--Bhatt2024} to compute reasonably accurate proxy in linear complexity $O(n)$. 
In the experiment below, we used the pCRP$^\star$-LCB.
We defer the detailed discussion about implementation to Appendix~\ref{app:sec:implementation_lcbs}, including its fast, approximate version with $O(n)$ complexity.

\paragraph{On the gambling technique.}
At first glance, using the gambling technique may seem excessive when strong time-uniformity is not needed, especially given their looser bounds. However, in this paper, Cover’s UP plays a central role due to its principled adaptation to the process statistics \emph{without prior knowledge}, enabled by mixture wealth. Here, time-uniformity is a \emph{byproduct} of the gambling framework, not the main goal. The $\log n$ regret term reflects the cost of this adaptation. If additional information---e.g., bounds on variance or sub-Gaussianity---were known, UP might be unnecessary and the $\log n$ term potentially avoidable. Whether this dependence can be reduced remains open.

\subsection{Off-Policy Selection with Pessimism via semi-Unbounded-coin-Betting (PUB)}
We propose a selection strategy, termed \emph{\textbf{P}essimism by semi-\textbf{U}nbounded-coin-\textbf{B}etting} (PUB), as follows: apply the UP-LCB $\lcbup{n}{\dt}(\cdot)$ from Eq.~\eqref{eq:uplcb} or the \pcrplcb{} $\lcbpcrp{n}{\dt}(\cdot)$ from Eq.~\eqref{eq:pcrplcb} to the underlying processes $\{\ips_{1:n}^\pi\}_{\pi \in \Pi}$, and select the policy with the highest lower confidence bound:
\begin{align}
\selectup\defeq \arg\max_{\pi\in\Pi} \lcbup{n}{\dt}(\ips_{1:n}^\pi)
\qquad\text{and}
\qquad
\hpi_{\mathsf{pCRP}^\star}^{(\dt)} \defeq \arg\max_{\pi\in\Pi} \lcbpcrp{n}{\dt}(\ips_{1:n}^\pi).
\label{eq:up_select}
\end{align}
Hereafter, we will omit the superscript $^{(\dt)}$.
The following guarantee is immediate from Theorem~\ref{thm:up_lcb_rate}.
\begin{theorem}[Selection]
\label{thm:main_selection} 
Let $\dt'=|\Pi|/\dt$.
With probability $\ge 1-2\dt$, 
for any $\pi^*\in \Pi$ and $\hpi\in \{\selectupsimple,\selectpcrpsimple\}$, we have\vspace{-.25em}
\begin{align*}
0\le \mean(\pi^*) - \mean(\hpi) 
&\le \sqrt{\fr{48\tilde{\var}(\pi^*)}{n} F_n^{(\dt')}
} \vee \fr{12\mean(\pi^*)}{n} F_n^{(\dt')}.
\end{align*}
Moreover, if 
$n\ge 108\bigl(1\vee 36\fr{\mean(\pi^*)^2}{\til{\mathsf{V}}(\pi^*)}\bigr) F_n^{(\dt')}$,
for $b_n^{(\dt)}\defeq \sqrt{\fr{\mu^2}{2\sigma^2}\fr{F_n^{(\dt)}}{n}}$, we further have
\begin{align}
\mean(\pi^*) - \mean(\hpi) 
&\le 
\inf_{b\in(0,\fr34]}
\biggl\{
\frac{b}{\mean(\pi^*)}
\WW_b[\ips_1^{\pi^*}]
+\fr{\mean(\pi^*)}{b}
\fr{F_n^{(\dt')}}{n}
\biggr\}
\le 2\sqrt{
\frac{F_n^{(\dt')}}{n}
\WW_{b_n^{(\dt')}}[\ips_1^{\pi^*}]}.
\label{eq:selection_refined}
\end{align}
\end{theorem}

We compare the guarantee with that of the Logarithmic Smoothing (LS) estimator of \citet{sakhi24logarithmic}; see the definition in Appendix~\ref{app:sec:estimators}.
In Proposition~6 therein, it is proved that the estimator achieves the regret bound $\beta \EE\Bigl[\fr{(\ips_1^{\pi^*})^2}{1+\beta\ips_1^{\pi^*}}\Bigr] + \fr{2}{\beta n}\ln\fr{2|\Pi|}{\dt}$. First, we note that, our first term inside the infimum can be viewed as a \emph{centered} version of their first term. 
Also, similar to that the first term of their regret is always bounded by $\beta \mean(\pi^*)$, 
our first term is also bounded by $\frac{b}{\mean(\pi^*)}\WW_0[\ips_1^{\pi^*}]\wed \WW_1[\ips_1^{\pi^*}]$.
In the second part of the statement, we further show in Eq.~\eqref{eq:selection_refined} that, for $n$ sufficiently large, the regret scales as $\tilde{O}(\fr{1}{\sqrt{n}})$, where the leading factor is $\sqrt{\WW_{b_n^{(\dt')}}[\ips_1^{\pi^*}]}$, which can be rewritten as
\[
\sqrt{\WW_{b_n^{(\dt')}}[\ips_1^{\pi^*}]}
= \sqrt{\EE\biggl[
\fr{(\ips_1^{\pi^*}-\mean(\pi^*))^2}{1+O(\sqrt{1/n})(\ips_1^{\pi^*}-\mean(\pi^*))}
\biggr]}
\quad
\text{vs.}
\quad
\underbrace{1 + {\EE\biggl[\fr{(\ips_1^{\pi^*})^2}{1 + O(\sqrt{1/n}) \ips_1^{\pi^*}}\biggr]}}_{\text{\citet[Proposition~6]{sakhi24logarithmic}}}.
\]
As noted in Figure~\ref{fig:comparison}, we note that $\sqrt{\WW_{b_n^{(\dt')}}[\ips_1^{\pi^*}]}\approx \til{\mathsf{V}}(\pi^*)$ is strictly smaller than $1+\EE[(\ips_1^{\pi^*})^2]$ for $n$ sufficiently large.
Arguably the most appealing property of our selection method is that we achieve these bounds in a \emph{parameter-free} sense, unlike the existing methods that require tuning $\beta$.
This is implemented by the infimum in the bound unlike the LS estimator, which shows that our estimator automatically adapts to the ``optimal'' hyperparameter, as a consequence of applying the wealth of Cover's UP or penalized-best-CRP.
Note that the price of adaptivity is only a logarithmic factor.

\section{Off-Policy Learning}
\label{sec:learning}
As a natural extension of the betting-based method for OP learning, we can formulate an optimization problem via Lagrange multipliers as follows:
\begin{align*}
  \max_{\pi\in\Pi} \max_{\alpha \ge 0} \min_\nu \Biggl\{\nu + \alpha\biggl(\max_{b\in[0,1]} \sum_{t=1}^n \ln\Bigl(1 + b \fr{\ips_t^{\pi}-\nu}{\nu}\Bigr) - \ln\fr{\sqrt{\pi(n+1)}}{\dt/|\Pi|}\biggr)
  \Biggr\}.
\end{align*}
However, this optimization is not straightforward to implement in practice.
Thus motivated, in this section, we consider a broad class of \emph{pessimistic} objective functions that take the following simple form, involving a \textit{score function} $\phi\colon\mathbb{R}_+\to\mathbb{R}$, where $\beta\ge 0$ is a hyperparameter:
\begin{align}\label{eq:op-learning-alg}
    \hpi_n \defeq \arg \max_{\pi \in \Pi} \sum_{t=1}^n %
    \phi (\beta \ips_t^\pi).
\end{align}
This form of objective admits a practical optimization with stochastic-gradient-based algorithms.
To guarantee statistical efficiency, we further restrict our attention to the following assumption:
\begin{assumption}\label{ass:phi-new}
For some $c_1, c_2 \in (0,1]$,
  a score function $\phi:\RR_+\rarrow\RR$ satisfies
  \begin{align*}
      -\ln\biggl(1 - x + \fr{x^2}{c_1+c_2x}\biggr) \le \phi(x) \le \ln(1+x).
  \end{align*}
\end{assumption}
Below, we provide a few concrete examples from this family.
\begin{proposition}[Examples of score functions]
\label{prop:phi}
  The following satisfy Assumption~\ref{ass:phi-new}:
  \begin{itemize}
    \item Logarithmic smoothing~\citep{sakhi24logarithmic}: $\phi^{\text{LS}} (x)=\ln(1+x)$ with $c_1 = c_2 = 1$.
    \item Clipping: $\phi^{\text{clipping}} (x) = \ln(1 + (x \wed 1))$ with $c_1 = c_2 = \fr12$.
    \item Freezing: $\phi^{\text{freezing}}(x) = \ln(1 + x\cd \onec{x \le 1})$ with $c_1 = c_2 = \fr12$.
  \end{itemize}
\end{proposition}

\begin{wrapfigure}{r}{0.38\linewidth}
  \centering%
  \vspace{-.25em}%
  \includegraphics[width=.95\linewidth]{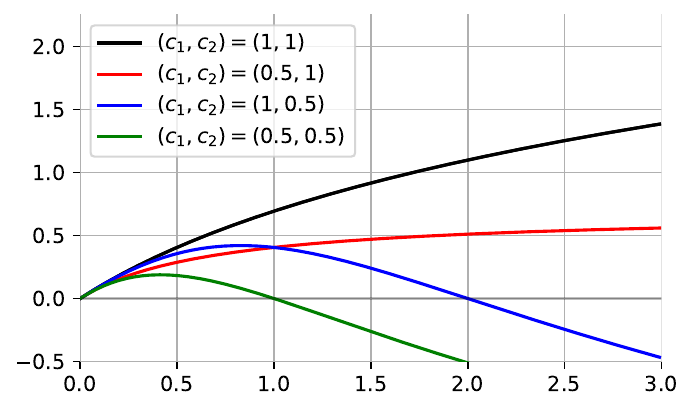} %
  \vspace{-.75em}
  \caption{Examples of the score function $\phi$ in Assumption~\ref{ass:phi-new}.
  }
  \label{fig:phi-second-order}
  \vspace{-2.5em}
\end{wrapfigure}
The clipping score function simply truncates the score function at $\ln 2$.
Note that the clipping is applied directly to $\ips^\pi$, rather than applied to $w_t^\pi$.
Penalizing large values of $\ips^\pi$, clipping may help reduce variance and improve sample efficiency by implementing a more aggressive pessimism.

The freezing score function implements an even more aggressive pessimism by zeroing out $\ips^\pi$ higher than $1$.
The potential benefit of freezing is to effectively remove samples $(x_t,a_t,r_t)$ for policies $\pi$ whose $\ips_t^{\pi}$ is too large.
This may have an even higher degree of variance reduction effect.

While \citet{sakhi24logarithmic} also proposed a family of choices for implementing pessimism, their condition only guarantees \textit{correctness} of the pessimism~\cite[Corollary 4]{sakhi24logarithmic}, but without \textit{sample efficiency}.
Their sample efficiency result is only proved for the logarithmic smoothing.
In contrast, we establish a sample-efficiency guarantee for a broad class of score functions.

\paragraph{Main Result.}
In the following, we show a smoothed second-order bound for a fairly large class of score functions satisfying Assumption~\ref{ass:phi-new}. Notably, our guarantee only depends on the optimal policy $\pi^\star$. The proof is deferred to Appendix~\ref{app:sec:proofs_learning}.
\begin{theorem}[Learning]
\label{thm:learning-new}
  Let $\hpi_n$ denote the estimator defined in Eq.~\eqref{eq:op-learning-alg} with a score function $\phi$ satisfying Assumption~\ref{ass:phi-new}, and let $\pi^\star$ be the optimal policy. 
  Then, with probability $\ge 1-\delta$, we have
  \begin{align}
  \label{eq:learning}
    \mean(\pi^\star) - \mean({\hpi_n}) 
    &\le \beta \EE\biggl[\fr{(\ips^{\pi^*})^2}{c_1 + c_2 \ips^{\pi^*}}\biggr] 
    + \fr{2}{\beta n} \ln\fr{|\Pi| }{\dt}
    - F_\beta(\phi)
    ,
  \end{align}
where we define the functional $F_\beta(\phi)\defeq \fr1{\beta} \ln(\EE[e^{\phi(\beta \ips_1(\hpi_n)) - \EE  [\beta \ips_1(\hpi_n)]}] )$ as the \emph{negative influence} induced by $\phi$, and it satisfies $F_\beta(\phi) \ge 0$.\looseness=-1
\end{theorem}
As a special case, our theorem recovers the guarantee of logarithmic smoothing~\cite[Proposition 6]{sakhi24logarithmic}.
The main term $\EE\Bigl[\fr{(\ips^{\pi^*})^2}{c_1 + c_2 \ips^{\pi^*}}\Bigr]$ is a smoothed second-moment term that specializes to that of~\citet{sakhi24logarithmic} when $c_1 = c_2 = 1$.
Thus, our bound inherits all the benefits such as being strictly better than IX~\citep{gabbianelli24importance} and being bounded from above with probability 1.

Additionally, different choices of \(\phi\) induce a nontrivial tradeoff in the resulting bound, e.g.,
\begin{align*}
  F_\beta(\phi^{\text{freezing}}) \le F_\beta(\phi^{\text{clipping}})\le F_\beta(\phi^{\text{LS}}).
\end{align*}
Thus, by using a score function with $c_1,c_2<1$ (e.g., freezing or clipping), we induce a larger negative influence, which may lead to improved performance in practice---especially in the small-sample regime---as we demonstrate in our experiments below.
We believe this tradeoff arises from a delicate balance between bias and variance.  
A more precise characterization is left for future work.

Finally, we note that 
$\beta$ is a hyperparameter that can be tuned using a holdout set, in conjunction with our proposed selection method from the previous section.

\section{Experiments}
\label{sec:exp}
We demonstrate the efficacy of our ideas, betting and freezing, under a synthetic, controlled experiment setup.
We first demonstrate the qualitative advantage of the UP-LCB against the existing baselines under heavy-tailed distributions. 
We then closely follow the setting of \citet{wang24oracle}, to which we refer for detailed descriptions.
We first present the learning experiment, followed by the selection experiment, in which we use the best policies from the learning phase across baselines.

\subsection{Synthetic Evaluation of UP-LCB under Heavy-Tail Distribution}
\label{sec:counter-example}

Here we empirically show that the betting-based LCB is not only statistically valid, but also converges to the target parameter in  a stable manner even for heavy-tailed data. 
We construct a synthetic contextual bandit setting to demonstrate such robustness as follows.
For a countably infinite context space $\cX =\{1,2,\ldots\}$ and a binary action space $\mathcal{A}=\{1,2\}$, consider:
(1) context distribution: $p(x=i) = \frac{6}{\pi^2} \frac{1}{i^2}$ for $i\in\mathbb{N}$;
(2) behavior policy: $\piref(a|x=i)=\mathrm{Bern}(a|\frac{1}{i^\beta})$; 
(3) reward distribution: $p(r=1|x=i,a=1)=1$,
$p(r|x=i,a=0)=\mathrm{Bern}(r|1-\frac{1}{i})$.
We can show that the fourth raw moment $\EE[(\ips_t^\pi)^4]$ does not exist; see Proposition~\ref{prop:counter_example} in Appendix for a formal statement.

We simulated the environment using $\beta=3$, where a higher $\beta$ leads to a heavier tail behavior of $\tilde{r}_t^\pi$, and thus a more erratic behavior for the existing baselines.
We generated the trajectory of interactions of length $n=10^4$ for $N=100$ random trials, and visualize in Figure~\ref{fig:unbdd_up_eb_main} the sample mean trajectory, the LCB based on the empirical Bernstein (EB) of \citet{maurer09empirical}, and a betting-based LCB proposed by \citet{Waudby-Smith--Wu--Ramdas--Karampatziakis--Mineiro2022} (see Appendix~\ref{app:sec:wswrkm} for its definition), Logarithmic Smoothing (LS) of \citet{sakhi24logarithmic}, and the UP-LCB, all averaged over the random trials. 
The shaded areas indicate empirical 10\% and 90\% quantiles.
Here, following \citep{wang24oracle}, for the EB-LCB, we used $\hat{\mean}_t(\pi)-\sqrt{2\hat{\var}_t(\pi)\ln\frac{2}{\dt}}$,
where $\hat{\mean}_t(\pi)$ and $\hat{\var}_t(\pi)$ are the empirical mean and variance of the importance weighted rewards $(\ips_i^\pi)_{i=1}^t$.

\begin{wrapfigure}{r}{0.6\linewidth}
    \vspace{-1.em}
    \includegraphics[width=\linewidth]{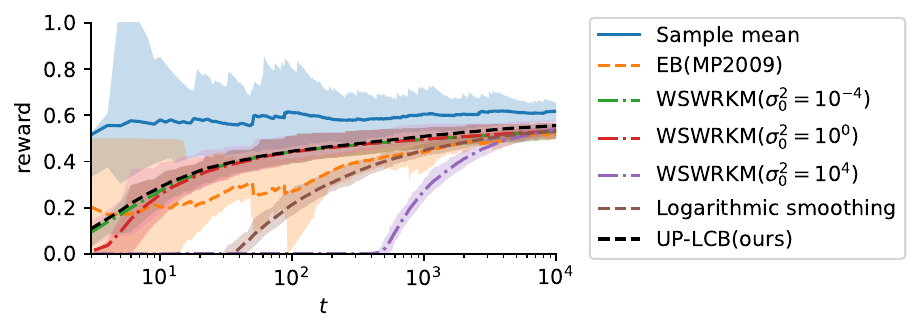}
    \vspace{-3em}
    \caption{Comparison of UP-LCB with baselines. %
    }    
    \label{fig:unbdd_up_eb_main}
  \vspace{-.5em}
\end{wrapfigure}
As shown in Figure~\ref{fig:unbdd_up_eb_main}, the UP-LCB provides a stable lower estimate of the target mean despite the heavy tail.
The estimates from the EB-LCB present erratic behaviors whenever we encounter a sample from the heavy tail.
The LS-LCB is also suboptimal as expected, being unable to adapt to the underlying variance.
The WSWRKM-LCB has a hyperparameter $\sigma_0^2$, which we set to $\{10^{-4},1,10^4\}$. 
WSWRKM-LCBs with $\sigma_0^2 \in \{10^{-4}, 1\}$ both show a fairly close or almost equal performance to UP-LCB on average. %
However, these two WSWRKM-LCBs exhibit a larger variability of $\sim 7\times10^{-2}$ measured by the difference between the 10\% and 90\% sample quantiles, compared to that of $\sim 4\times10^{-2}$ of WSWRKM-LCB.
On the other hand, for WSWRKM-LCB with $\sigma_0^2=10^4$, the variability was similar to UP-LCB.
This shows that the hyperparameter $\sigma_0^2$ trades off variability with the tightness of the confidence bound.
We also recall that the WSWRKM-LCB does not have a sample-efficiency guarantee.
We defer a more extensive discussion on the WSWRKM-LCB to Appendix~\ref{app:sec:wswrkm}.
In Appendix~\ref{app:sec:counter-example}, we include some realizations of the experiments to further demonstrate the actual behavior for erratic trajectories.\looseness=-1

\subsection{Off-Policy Learning}

\paragraph{Datasets.} 
The contextual bandit data were simulated using three multi-class classification datasets from OpenML~\citep{OpenML2013}.
Each dataset has $10^6$ data points.
Other statistics of the datasets are summarized in Table~\ref{tab:datasets} in Appendix.
For each learning method, we viewed each dataset as a multi-class regression problem, where each class corresponds to an action. We then treated a classifier, which maps a feature to a probability vector, as a deterministic policy that chooses the action of the maximum probability.
Among various configurations considered in \citep[Section~4]{wang24oracle}, we specifically considered the real-valued cost function and a single logging policy $\pi_{\text{good},\eps=0.1}$ therein.
This logging policy is defined as a random mixture of a deterministic policy (induced from a separately trained classifier) and a uniform-random policy, where $\eps=0.1$ defines the probability of using the uniform-random policy.
In Appendix~\ref{app:sec:add_exp}, we report additional results with two additional policies $\pi_{\text{good},\eps=0.01}$ and $\pi_{\text{bad},\eps=0.1}$ as done in \citep{wang24oracle}, where the latter combines a \emph{badly} trained deterministic classifier with the uniform-random policy.
We tested different fractions $\{0.01,0.1,1\}$ of datasets for training.

\paragraph{Baselines.}
We consider seven different methods in the learning experiment.
The default baseline is the minimizer of the IW estimators without any regularizer (denoted as \texttt{IW}).
We then consider a naive method (\texttt{Naive}, which naively maximizes $\sum_{t=1}^n \pi(a_t|x_t) r_t$ ignoring $\piref(a_t|x_t)$),
pseudo-loss (\texttt{PL})~\citep{wang24oracle}, clipped IW (\texttt{ClippedIW}), 
Implicit Exploration (\texttt{IX})~\citep{gabbianelli24importance}, Logarithmic Smoothing (\texttt{LS})~\citep{sakhi24logarithmic}, and lastly LS with freezing (\texttt{LS+freezing}), which we propose in Section~\ref{sec:learning}. We include explicit definitions of the estimators in Appendix~\ref{app:sec:estimators}.
For optimization (i.e., learning), we used the linear regression approach. 
After training, we computed
the relative improvement of each estimator against the \texttt{IW} baseline:
$\texttt{(relative improvement of $\pi$)}\defeq \frac{\hat{\mean}(\pi_{\texttt{IW}})-\hat{\mean}(\pi)}{\hat{\mean}(\pi_{\texttt{IW}})}$.
Here, we used the IW estimator to estimate the value $\hat{\mean}(\pi)$ of each policy $\pi$.
Similar to \citet{wang24oracle}, the hyperparameter $\beta$ in each estimator (see Appendix~\ref{app:sec:estimators}) were tuned based on our PUB method with $\dt=0.1$, using a 50/50 split of the data.
For each learning method, we swept the hyperparameter $\beta$ over eight values $\{0, 0.001, 0.003, 0.01, 0.03, 0.1, 0.3, 1\}$, and thus there are $48=6\times 8$ instances of methods in total.
For each dataset,
we ran each experiment with 50 different seeds. 

\paragraph{Results.}
The results are summarized in Figure~\ref{fig:learning}. 
As predicted by our analysis in Section~\ref{sec:learning}, \texttt{LS+freezing} (light blue) consistently outperforms \texttt{LS} (blue) in the small-sample regime and remains competitive more broadly. 
We attribute this to reduced variance from aggressive freezing, which effectively filters out outliers, an advantage that is especially pronounced when data is limited.

\begin{figure*}[hbt]
\centering
\includegraphics[width=0.8\linewidth]{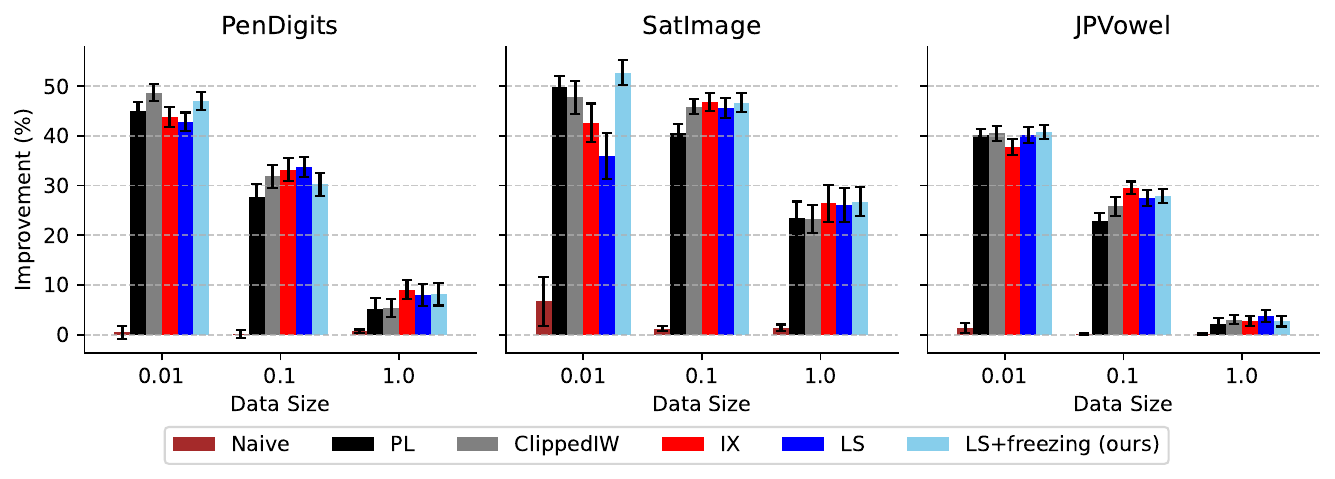}
\vspace{-.5em}
\caption{Results from the OP learning experiment, showing relative improvement of each method against the no-pessimism baseline.
We highlight the nearly consistent improvement of \texttt{LS+freezing} over \texttt{LS}.
}\vspace{-1em}
\label{fig:learning}
\end{figure*}

\subsection{Off-Policy Selection}
\paragraph{Setup.}
We reused the synthetic bandit data from the OP learning experiment.
For selection, we considered all the $7\times 8$ policies from the learning experiment as a policy class.
For each dataset and each run, we selected the best method using \texttt{PUB}, \texttt{LS}, \texttt{WSWRKM} (pessimism with WSWRKM-LCB using $\sigma_0^2=1$), and the original empirical Bernstein (\texttt{EB}) of \citet[Theorem 4]{maurer09empirical}, all with $\dt=0.1$.
Here, we aim to simulate the scenario where the practitioner tries out various policies with various hyperparameters on the training data and then chooses the best policy using the validation set.\looseness=-1

\paragraph{Evaluation.}
We again computed the relative performance improvement of each selected policy against that of the \texttt{IW} baseline.
To avoid misleading conclusion due to randomness, we performed the paired $t$-test for all pairs of the selection policies over the 50 random trials.
We report the best-performing selection method and indicate statistically indistinguishable selection methods (with a $p$-value > 0.05) in boldface.

\paragraph{Results.}
The results are summarized in Table~\ref{tab:selection}.
Remarkably, for all cases, the proposed method \texttt{PUB} performs the best, or is statistically indistinguishable from the best.
This demonstrates that \texttt{PUB} is not only statistically valid, but also perform empirically well, corroborating the practical benefit of variance-adaptive guarantee.
We also note that, despite the competitive performance of WSWRKM-LCB in Section~\ref{sec:counter-example}, this experiment reveals a failure mode of \texttt{WSWRKM}, particularly in the relatively large-sample regime. Additional OP selection results in Appendix~\ref{app:sec:add_exp} highlight even more severe failure cases.

\begin{table}[htb]\vspace{-1em}
  \caption{
    Summary of OP selection experiment. Size is the fraction of data used for training. The best is \textbf{boldfaced} for each column, and those which do not pass paired t-test with the best (i.e., those that are not statistically distinguishable from the best) at significance level 0.05 are \underline{underlined}. 
    PUB is either the best or indistinguishable from the best.
    }
    \centering
  \scalebox{0.85}{
  \begin{tabular}{c|ccc|ccc|ccc}
    \toprule
    Dataset & \multicolumn{3}{c|}{PenDigits} & \multicolumn{3}{c|}{SatImage} & \multicolumn{3}{c}{JPVowel} \\
    \midrule
    Size & 0.01 & 0.1 & 1 & 0.01 & 0.1 & 1 & 0.01 & 0.1 & 1 \\
    \midrule
    \texttt{EB} & \underline{24.67} & {\textbf{35.83}} & \textbf{11.66}
    & \underline{34.00} & \underline{35.41} & \textbf{24.03}
    & \underline{25.96} & \textbf{29.75} & \underline{1.91} \\
    \texttt{LS} & \underline{22.82} & \underline{35.22} & \underline{9.99}    
    & 31.80 & \underline{30.29} & \textbf{24.03}
    & \underline{24.39} & \underline{27.64} & \underline{1.65} \\
    \texttt{WSWRKM} & \textbf{36.55} & \underline{29.45} & 2.19 
    & \textbf{40.75} & \textbf{46.24} & 15.68
    & \underline{20.39} & 9.73 & \underline{1.76} \\
    \texttt{PUB}  (ours) & \underline{34.80} & \underline{31.69} & \underline{7.81}
    & \underline{37.10} & \underline{35.41} & \textbf{24.03}
    & \textbf{33.77} & \underline{20.82} & \textbf{2.08} \\
    \bottomrule
  \end{tabular}
  }
  \label{tab:selection}
\end{table}\vspace{-1.5em}

\section{Conclusion}
In this paper, we have established new state-of-the-art bounds for off-policy (OP) problems.
Our work opens up several interesting research directions.
First, since we have developed a selection method that adapts to the variance of the IW estimator, a natural next step is to characterize the optimal rate for offline regret.  
Second, it is worth investigating whether similar or stronger bounds can be achieved for a doubly robust estimator.
Third, it remains an open question whether one can design an objective that is amenable to optimization  
while matching the statistical rate of our offline selection method.
Last but not least, our LCB for nonnegative random variables is, to our knowledge, the strongest in the literature---particularly in terms of convergence rate---since it requires only the existence of variance, not higher-order moments. Exploring its potential applications or extensions to other learning-theoretic problems could lead to improved guarantees and deeper insights.

\newpage

\acks{KS was supported in part by the National Science Foundation under grant CCF-2327013 and Meta Platforms, Inc.}
\bibliography{ref}

\clearpage
\appendix

\addtocontents{toc}{\protect\StartAppendixEntries}
\listofatoc
\vspace{1em}
\hrule

\section{Related Work}
\label{sec:related}

Since the work of \citet{swaminathan15batch}, there have been numerous studies on off-policy contextual bandits and reinforcement learning.
An exhaustive literature review with a detailed comparison would warrant a separate survey paper.
Here, we focus on categorizing representative works from a theoretical perspective based on the level of guarantees they provide.

\paragraph{No Finite-Time Correctness Guarantee.}
Methods in this category do not provide a provable guarantee of correctness for the proposed confidence bound on the performance of a given policy under evaluation. Notable examples include the empirical likelihood approach~\citep{karampatziakis20empirical} and the self-normalized estimator~\citep{swaminathan15batch}. These lack explicit finite-time correctness guarantees, let alone sample efficiency guarantees. Moreover, coverage violation of \citet{karampatziakis20empirical} was empirically observed in \citet[Figure~3]{kuzborskij21confident}.

\paragraph{Finite-Time Correctness Without Sample Efficiency Guarantee.}
Several approaches only come with a finite-sample correctness guarantee of the proposed confidence bound, but without a convergence rate guarantee, and consequently with no offline regret guarantee.
This includes the seminal work of \citet{london19bayesian} leveraging PAC-Bayesian bounds, exponential weighting~\citep{aouali23exponential}, empirical Bernstein style bound~\citep{sakhi23pac}, Efron-Stein semi-empirical bound for the self-normalized importance weight~\citep{kuzborskij19efron,kuzborskij21confident}, and betting-based bounds~\citep{karampatziakis21off,Waudby-Smith--Wu--Ramdas--Karampatziakis--Mineiro2022}.

\paragraph{Sample Efficiency Guarantee Under Bounded Probability Ratios.}
Including many works mentioned above, several works have assumed a finite upper bound on the weights $w_{1:n}^\pi$.
Many works mentioned above make this assumption.
Of those that provide sample efficiency guarantees, the following works either assume bounded weight or their guarantees become vacuous when the weight is unbounded: \citet[Corollary 4.3]{jin22policy}, \citet{wang24oracle}, and \citet{zenati23sequential}.

\paragraph{Sample Efficiency Guarantee Without Bounded Probability Ratios.}
Only recently have methods with sample efficiency guarantees that remain valid without the bounded probability ratio assumption been proposed. These methods allow the behavior policy to assign arbitrarily small probabilities to certain actions.
While such bounds can still be vacuous in the worst case, they may remain meaningful even when $\piref(a\mid x)$ approaches zero, depending on the distribution of the context $x$ and the reward function.
Early studies in this direction established sample efficiency guarantees that depend on empirical quantities, such as those in \citet[Theorem 4.1]{jin22policy}. However, these guarantees are challenging to interpret and compare with other bounds, as they depend on the specific randomness in the bound’s construction.
In a seminal work, \citet{gabbianelli24importance} provided the first deterministic sample efficiency bound, which was later improved by \citet{sakhi24logarithmic}.
Our work falls into this category, achieving the strongest sample efficiency guarantees for selection and evaluation while matching the bound of \citet{sakhi24logarithmic} for learning.

\section{Deferred Discussions and Proofs for Off-Policy Selection}

\subsection{Implementation of Proposed LCBs}
\label{app:sec:implementation_lcbs}

In this section, we discuss the implementation of the proposed LCBs, \ie UP-LCB and pCRP$^\star$-LCB, in detail and their complexity; see Table~\ref{tab:complexity} for a summary.
We make a distinction of the \emph{online complexity} (when constructing LCB at each time step) and \emph{offline complexity} (when constructing LCB only for the last time step).
We also provide a computationally efficient version, LBUP-LCB, based on a similar trick of \citet{Ryu--Bhatt2024}.

\begin{table}[htb]
    \centering
    \caption{Comparison of complexity of different LCBs. 
    We present the time complexity to compute a LCB for each time step in ``Online Complexity'' for a length-$n$ trajectory, and that to compute a LCB only for the last step with $n$ samples in ``Offline Complexity''.
    Here, $M$ denotes the maximum time complexity for a root finding procedure. For example, if we use a bisect algorithm with target precision $\eps$, $M=O(\ln \fr{1}{\eps})$.}
    \begin{tabular}{c c c c}
    \toprule
    Algorithm & Online complexity & Offline complexity & Rate guarantee\\
    \midrule
    UP-LCB & $\Th(Mn^2)$ & $\Th(n^2 + Mn)$ & Yes\\
    pCRP$^\star$-LCB & $\Th(M^2n^2)$ & $\Th(M^2n)$ & Yes\\
    \midrule
    LBUP-LCB & $\Th(Mn)$ & $\Th(n+M)$ & No \\
    \bottomrule
    \end{tabular}
    \label{tab:complexity}
\end{table}

\newcommand{\ct}{\til{c}}
\newcommand{\diff}{\mathrm{d}}
\newcommand{\rhob}{\boldsymbol{\rho}}
\newcommand{\Natural}{\mathbb{N}}
\newcommand{\Real}{\mathbb{R}}

\subsubsection{Computing UP-LCB with Dynamic Programming}
As alluded to earlier, we can compute the exact UP wealth using dynamic programming. 
A similar statement was proved by \citet{Ryu--Bhatt2024} in the context of confidence sequences for bounded stochastic processes, and the original dynamic programming argument for UP can be found in \citep{Cover--Ordentlich1996}.

Recall that the wealth of UP is defined as a mixture wealth of CRPs:
\begin{align}
\nonumber
\wealthup_t(y_{1:t};\nu)
\defeq \int_0^1 \wealthcrp{b}_t(y_{1:t};\nu) w(b) db.
\end{align}
The following proposition holds for any weight distribution $w(b)$.
\begin{proposition}
\label{thm:exact_up}
The wealth of UP can be computed as
\begin{align}
\wealthup_t(y_{1:t};\nu)
=\sum_{k=0}^t 
\fr{1}{\nu^k}\psi_t^{w}(k)
y^{(t)}(k),
\label{eq:up_wealth_expression}
\end{align}
where we 
define
\begin{align}
\psi_t^{w}(k) &\defeq \int_0^1 b^k(1-b)^{t-k}\diff w(b),
\label{eq:mixture_wealth_individual}\\
y^{(t)}(k)&\defeq \sum_{x_{1:t}\in\{0,1\}^t\suchthat k(x_{1:t})=k} \prod_{i=1}^t y_i^{x_i},\label{eq:seq_k_statistics}
\end{align}
and $k(x_{1:t})\defeq \sum_{i=1}^t x_i$.
Furthermore, for each $t\ge 1$, we have
\begin{align}
y^{(t)}(k)=
\begin{cases}
y^{(t-1)}(0) & \text{if }k=0,\\
y_t y^{(t-1)}(k-1) + y^{(t-1)}(k) & \text{if }1\le k\le t-1\\
y_t y^{(t-1)}(t-1) & \text{if }k=t.
\end{cases}
\label{eq:up_recursive_update}
\end{align}
\end{proposition}

\begin{proof}
We first note that we can write the cumulative wealth of any constant bettor $b$ as
\begin{align}
\wealthcrp{b}_t(y_{1:t};\nu)
&=\sum_{x_{1:t}\in\{0,1\}^t} \prod_{i=1}^t \Bigl(\fr{y_i b}{\nu}\Bigr)^{x_i} (1-b)^{1-x_i},
\label{eq:distributive_law}
\end{align}
where the equality follows by the distributive law.
To see Eq.~\eqref{eq:up_wealth_expression}, we first note that continuing from Eq.~\eqref{eq:distributive_law}, we have
\begin{align}
\wealthcrp{b}_t(y_{1:t};\nu)
&= \sum_{k=0}^t 
\nu^{-k} b^k(1-b)^{t-k} 
\sum_{x_{1:t}\in\{0,1\}^t\suchthat k(x_{1:t})=k} \prod_{i=1}^t  y_i^{x_i}\nonumber\\
&= \sum_{k=0}^t 
\nu^{-k} b^k(1-b)^{t-k} 
y^{(t)}(k),
\label{eq:wealth_const_betting}
\end{align}
and thus integrating over $b$ with respect to $w(b)$ leads to \eqref{eq:up_wealth_expression}.
The recursive update  in Eq.~\eqref{eq:up_recursive_update} is straightforward.
\end{proof}

This proposition shows that, the recursive update takes $O(t)$ at time step $t$, and thus the online complexity is $O(Mn^2)$.
Even for the offline setting where we only need to compute the LCB with the entire samples once, we need to run the recursive update in Eq.~\eqref{eq:up_recursive_update} for each $t=1,\ldots,n$ and evaluating the wealth defined in Eq.~\eqref{eq:up_wealth_expression} takes $O(t)$, which leads to the complexity $O(n^2+Mn)$.

\subsubsection{Computing pCRP\texorpdfstring{$^\star$}{*} Wealth}

Recall that the pCRP$^\star$-LCB in Eq.~\eqref{eq:pcrplcb}
is defined by the (unique) root $\nu$ of the equation
\[
\wealthpcrp(Y_{1:n};\nu) 
= \frac{1}{\sqrt{\pi(n+1)}}
\sup_{b\in(0,1)}\wealthcrp{b}(Y_{1:n};\nu)
= \fr{1}{\dt}.
\]
Here, for each $\nu$, the maximizer $b$ can be found by finding the root of the derivative
$\fr{d}{db}\wealthcrp{b}(Y_{1:n};\nu)=0$.
Hence, we can numerically find the root by the bisect algorithm over both $\nu>0$ and $b\in(0,1)$. 
Note that the CRP wealth evaluation takes $O(t)$ at time step $t$, and thus computing the LCB takes $O(M^2t)$. Therefore, the online and offline complexities are $O(M^2n^2)$ and $O(M^2n)$, respectively.

\newcommand{\apporder}{r}

\subsubsection{Lower-Bound Universal Portfolio: A Fast Alternative}
\label{app:sec:lbup}
Adapting the development of \citet{Ryu--Bhatt2024} for $[0,1]$-valued random processes, here we present a fast alternative approach that tightly approximates the UP wealth. 
The idea is to directly compute a mixture of very tight lower bounds on the CRP wealths.
The mixture of lower bounds can be computed efficiently by numerical integration, by viewing the lower bound as an (unnormalized) exponential family distribution.
While there is no guarantee on the approximation error, the resulting bound is empirically a very good proxy to the UP-LCB, even better than the pCRP$^\star$-UCB, when sample size is sufficiently large; see Figure~\ref{fig:ex_betting_lcb_lbup}.

\paragraph{Tight Lower-Bound on CRP Wealth.}
We start with the following lemma from \citep{Ryu--Bhatt2024}.
We note that \citep{sakhi24logarithmic} also proved a similar statement (see Lemma~10 therein), but the domain is restricted to $\Real_+$ and thus not sufficient for our purpose.

\begin{lemma}[{\citealp[Lemma~25]{Ryu--Bhatt2024}}]
\label{lem:monotone}
For an integer $\ell\ge 1$, if we define
\[
f_\ell(t)\defeq
\begin{cases}
\displaystyle\frac{\ln(1+t)-\sum_{k=1}^{\ell-1}-\frac{(-t)^k}{k}}{\frac{(-t)^\ell}{\ell}} & \text{if $t>-1$ and $t\neq 0$,}\\
-1 &\text{if $t=0$},
\end{cases}
\]
then $t\mapsto f_\ell(t)$ is continuous and strictly increasing over $(-1,\infty)$.
\end{lemma}

We can then prove the following lower bound.
As noted in \citep{Ryu--Bhatt2024}, the positive integer $\apporder\ge1$ in the statement can be understood as the approximation order.
Empirical results show that a higher order $\apporder$ results in a tighter lower bound, but we do not have a formal proof.
\begin{lemma}
\label{lem:generalized_lower_bound}
For any $\apporder\in \Natural$, $b\in[0,1]$, and $z\ge 0$, we have
\begin{align*}
\ln(1-b +bz)
&\ge \sum_{k=1}^{2\apporder-1} \frac{b^k}{k} \{(1-z)^{2\apporder}
    -(1-z)^k\} + (1-z)^{2\apporder}\ln (1-b).
\end{align*}
\end{lemma}
\begin{proof}
Note that the right hand side diverges to $-\infty$ and thus the inequality becomes vacuously true for $b=1$.
We now assume that $b<1$, which ensures $b(z-1)\ge -b>-1$.
Hence, from Lemma~\ref{lem:monotone},
we have $f_{2r}(b(z-1)) \ge f_{2r}(-b)$, which is equivalent to
\begin{align*}
\frac{\ln(1+b(z-1))-\sum_{k=1}^{2r-1}-\frac{(-b(z-1))^k}{k}}{\frac{(-b(z-1))^{2r}}{2r}}
\ge 
\frac{\ln(1-b)-\sum_{k=1}^{2r-1}-\frac{(-b)^k}{k}}{\frac{(-b)^{2r}}{2r}}.
\end{align*}
Rearranging the terms concludes the proof.
\end{proof}

The lower bound in the statement can be understood as the logarithm of an unnormalized exponential family distribution over $z$, i.e.,
\begin{align}
\ln(1-b + bz)
&\ge 
\ln \psi_r(z|b),
\label{eq:long_eq2}
\end{align}
where $\psi_r(z|b)$ is an unnormalized exponential family distribution defined as
\begin{align}
\psi_r(z|b)
&\defeq\exp(\boldsymbol{\th}_r(b)^\intercal\bT_r(z)).
\end{align}
Here, $\boldsymbol{\th}_r(b)$ is the natural parameter defined as
\begin{align*}
\boldsymbol{\th}_r(b)
\defeq \begin{bmatrix}
b\\
{b^2}/{2}\\
\vdots\\
{b^{2r-1}}/{(2r-1)}\\
\ln (1-b)
\end{bmatrix},
\end{align*}
and
$\bT_r(z)$ is the sufficient statistics defined as 
\begin{align*}
\bT_r(z) \defeq \begin{bmatrix}
(1-z)^{2r} - (1-z)\\
(1-z)^{2r} - (1-z)^2\\
\vdots\\
(1-z)^{2r} - (1-z)^{2r-1}\\
(1-z)^{2r}
\end{bmatrix}
= \sum_{j=0}^{2r} \binom{2r}{j} (-1)^j z^j\begin{bmatrix}
1\\
1\\
\vdots\\
1\\
1
\end{bmatrix}
- \begin{bmatrix}
\displaystyle
\sum_{j=0}^1 \binom{1}{j} (-1)^j z^j\\
\displaystyle\sum_{j=0}^2 \binom{2}{j} (-1)^j z^j\\
\vdots\\
\displaystyle\sum_{j=0}^{2r-1} \binom{2r-1}{j} (-1)^j z^j\\
0
\end{bmatrix}.
\end{align*}
From this definition, it is easy to check that 
\begin{align*}
\prod_{t=1}^n \psi_r(z_t|b)
= \exp\Bigl(
\boldsymbol{\th}_r(b)^\intercal \sum_{t=1}^n \bT_r(z_t)
\Bigr),
\end{align*}
and $\sum_{t=1}^n \bT_r(z_t)$ is a function of $(s_j(z_{1:n}))_{j=0}^{2r}$, where we denote the (unnormalized) empirical $j$-th moment for $j\in\Natural$ by
\[
s_j(z_{1:n})
\defeq \sum_{t=1}^n z_t^j.
\]
This implies that the lower bound can be readily computed from the empirical moments, unlike the CRP wealth or UP wealth that requires storing the entire history $z_{1:n}$.

\newcommand{\alphab}{\boldsymbol{\alpha}}
\paragraph{Mixture of Lower-Bounds on CRP Wealths.}
For computational tractability, we now consider a mixture weight in the form of the \emph{conjugate prior} of $\psi_r(z|b)$, defined as
\begin{align}
w_r(b;\alphab)
\label{eq:general_prior}
&\defeq \frac{\exp(\boldsymbol{\th}_r(b)^\intercal\alphab)}{Z_r(\alphab)}.
\end{align}
Here, $\alphab\in\Real^{2r}$ is a hyperparamter of the conjugate prior, and
\[
Z_r(\alphab)\defeq \int_0^1 \exp(\boldsymbol{\th}_r(b)^\intercal\alphab) \diff b
\] 
is the \emph{partition function}.
By the following theorem, computing the mixture of the CRP wealths with respect to this conjugate prior only requires to compute the normalization constant efficiently:
\begin{theorem}
\label{thm:lbup}
Let $r\ge 1$.
For any $y_{1:t}\in\Real_{\ge 0}^t$ and $\nu>0$,
we have
\begin{align}
\int \wealthcrp{b}_n(y_{1:n};\nu) w_r(b;\alphab) \diff b
\ge\frac{
Z_r(\sum_{t=1}^n \bT_r(\fr{y_t}{\nu})+\alphab)}
{Z_r(\alphab)}.
\label{eq:lbup_wealth}
\end{align}
\end{theorem}
In the special case of $r=1$, we can compute $Z_1(\alphab)$ in an analytical form if $\alpha_1\ge 0$:
\begin{align}
Z_1(\alphab)= e^{\alpha_1}\alpha_1^{-\alpha_2-1}\gamma(\alpha_2+1,\alpha_1).
\label{eq:normal_cont_special}
\end{align} 
Here, $\gamma(s,x)\defeq \int_0^x t^{s-1}e^{-t}\diff t$ for $s>0$ denotes the lower incomplete gamma function.
For $r>1$, we need a numerical integration library to compute the partition function.

We note that the conjugate prior is not same as the beta prior of Cover's UP in general. 
In particular, however, if we set $\alphab=\boldsymbol{0}$, then the prior $w_r(b;\alphab)$ boils down the uniform distribution over $[0,1]$, and the resulting mixture wealth lower bound can be viewed as a lower bound to Cover's UP with the uniform prior (\ie $\mathsf{Beta}(1,1)$ prior).
Following \citet{Ryu--Bhatt2024}, we refer to the resulting wealth lower bound the \emph{lower-bound UP wealth of approximation order $r$}, or LBUP($r$) in short. We refer to the resulting LCB as the LBUP($r$)-LCB.

\paragraph{Implementation and Complexity.}
We can numerically compute the LBUP($r$)-LCB using the bisect method
Since we only need to keep track of the $2r$ empirical moments $(s_j(y_{1:t}))_{j=1}^{2\apporder}$, the storage complexity is $O(\apporder)$ and per-step time complexity for function evaluation is $O(\apporder)$ at any time step.
Consequently, for computing the LBUP($r$)-LCB, the online complexity is $O(Mnr)$ and the offline complexity is $O(n+Mr)$.

\paragraph{Simulation.}
We simulated the UP-LCB, pCRP$^\star$-LCB, and LBUP($r$)-LCB for $r\in\{1,2,3\}$ for the same synthetic setting used in Figure~\ref{fig:ex_betting_lcb}.
We generated $n=10^4$ i.i.d. Gamma random variables with shape and scale parameters of 6 and $1/8$, respectively, and thus of mean $3/4$.
The results are summarized in Figures~\ref{fig:ex_betting_lcb_lbup},~\ref{fig:lbup_relerr},~and~\ref{fig:lbup_time}.
In particular, we remark that LBUP($r$)-LCBs (especially with $r\ge 2$) very closely approximate the UP-LCB (Figure~\ref{fig:lbup_relerr}) better than pCRP$^\star$ in a large sample regime, exhibiting better scalability over $n$ (Figure~\ref{fig:lbup_time}).
We note, however, that we do not have a formal guarantee for the closeness of LBUP-LCB to UP-LCB, and LBUP-LCBs require some burn-in samples ($\sim 10^2$ samples in this example) to become sufficiently close to UP-LCB.
For an off-policy inference setting with large-scale data, practitioners may consider using LBUP-LCB if the sample trajectory is sufficiently long, and otherwise may prefer pCRP$^\star$ for guaranteed performance with moderate complexity.

\begin{figure}[t]
    \centering
    \includegraphics[width=.95\linewidth]{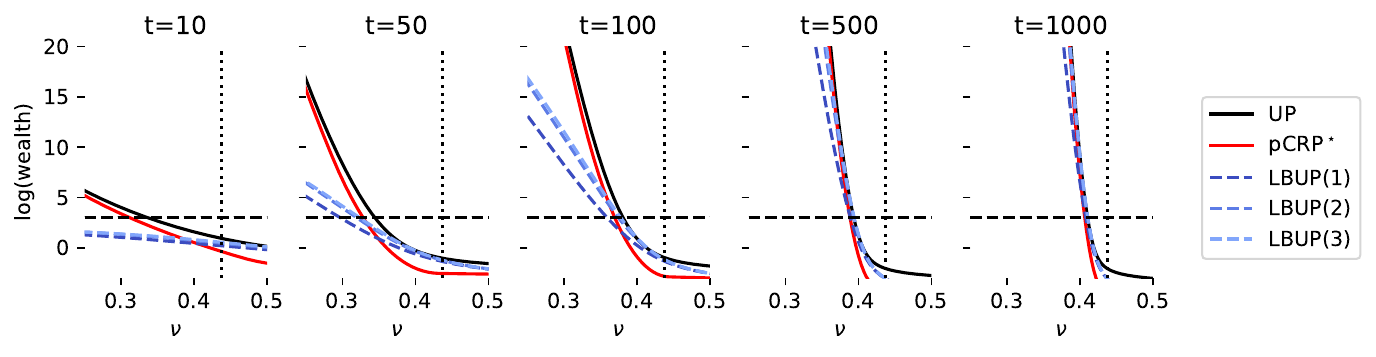}
    \vspace{-1em}
    \caption{Example of the evolution of cumulative wealths achieved by Cover's UP in Eq.~\eqref{eq:wealth_up}, and the penalized best CRP wealth in Eq.~\eqref{eq:up_vs_pcrp}, and the lower-bound universal portfolio in Appendix~\ref{app:sec:lbup}.
    The setting is exactly same as Figure~\ref{fig:ex_betting_lcb}, except that we use larger time steps and depict with a different range for $\nu$.
    }\label{fig:ex_betting_lcb_lbup}
\end{figure}

\begin{figure}[t]
    \centering
    \includegraphics[width=\linewidth]{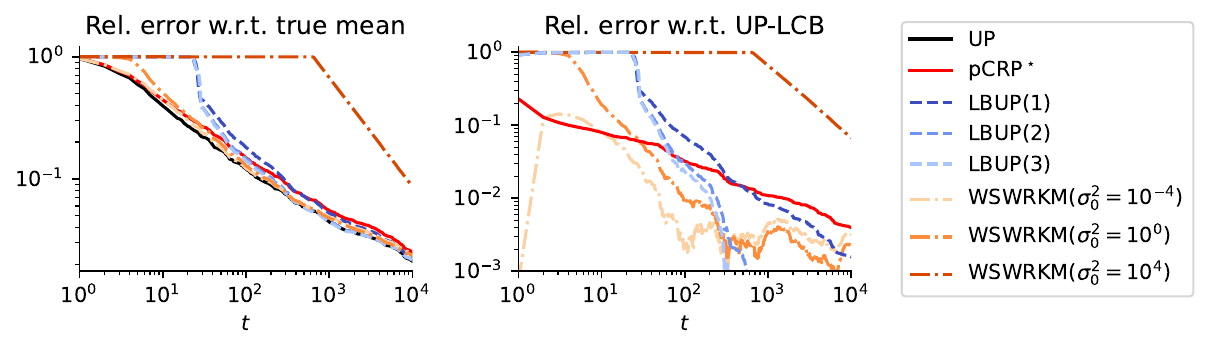}
    \vspace{-2em}
    \caption{Convergence UP-LCB, pCRP$^\star$-LCB, and LBUP($r$)-LCB for $r\in\{1,2,3\}$. 
    The left panel and right panel present the relative convergence of each LCB with respect to true mean and UP-LCB, respectively.
    WSWRKM refers to the method in \citep{Waudby-Smith--Wu--Ramdas--Karampatziakis--Mineiro2022}, whose definition and discussion of the result can be found in Appendix~\ref{app:sec:wswrkm}.}
    \label{fig:lbup_relerr}
\end{figure}

\begin{figure}[t]
    \centering
    \includegraphics[width=0.65\linewidth]{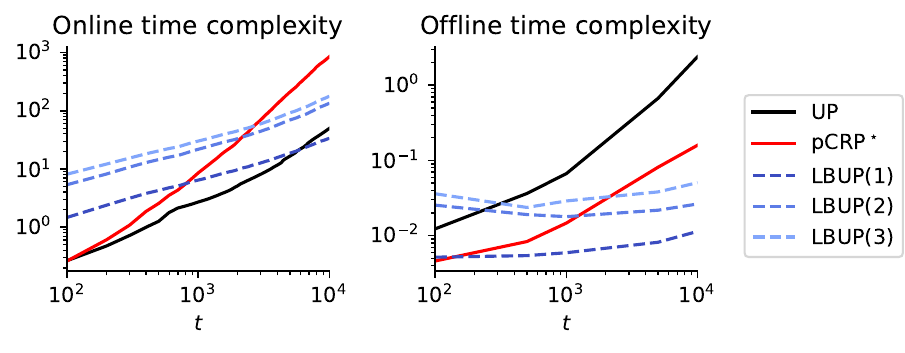}\vspace{-1em}
    \caption{Online and offline time complexity for computing UP-LCB, pCRP$^\star$-LCB, and LBUP($r$)-LCB for $r\in\{1,2,3\}$.}
    \label{fig:lbup_time}
\end{figure}

\begin{comment}
This is wrong, as the left hand side is expectation of log wealth, not the log of expected wealth.
\begin{theorem}
Define $s_j(y_{1:t})
\defeq \sum_{i=1}^t y_i^j$ for $y_{1:t}\in \Real_+^t$ for $j=1,\ldots,2\apporder$.
Define $c_k \defeq \fr{B(k+\fr12,\fr12)}{B(\fr12,\fr12)}$ for $k\in\Natural$ and $d\defeq \psi(\fr12)-\psi(1)=-2\ln 2$, where $\psi(x)$ denotes the digamma function.
For any $r\in\Natural$ and any $y_{1:n}\in\Real_+^n$, we have 
\begin{align*}
\ln \wealthup_t(y_{1:t};\nu)
&\ge \sum_{j=0}^{2r}\biggl\{
\Bigl(
\sum_{k=1}^{2r-1} \fr{c_k}{k} + d\Bigr) \binom{2r}{j} 
- \sum_{k=j\vee 1}^{2r-1}\fr{c_k}{k} \binom{k}{j}
\biggr\}  
\fr{s_j(y_{1:t})}{(-\nu)^j}.
\end{align*}
\end{theorem}
\begin{proof}
Plugging in $z\gets \fr{y_i}{\nu}$ into the inequality in Lemma~\ref{lem:generalized_lower_bound}, summing over $j=1,\ldots,t$, and and integrating over $b\in(0,1)$ with respect to the density $w(b)=\mathsf{Beta}(b;\alpha,\beta)$, we have
\begin{align*}
\ln \wealthup_t(y_{1:t};\nu)
&\ge
\sum_{j=0}^{2r}\biggl\{
\Bigl(
\sum_{k=1}^{2r-1} \fr{\E_w[X^k]}{k} + \E_w[\ln(1-X)] \Bigr) \binom{2r}{j} 
- \sum_{k=j\vee 1}^{2r-1}\fr{\E_w[X^k]}{k} \binom{k}{j}
\biggr\}  
\fr{s_j(y_{1:t})}{(-\nu)^j},
\end{align*}
where $X\sim \mathsf{Beta}(\alpha,\beta)$.
The claim follows by noting that $\E_w[X^k]=\fr{B(k+\alpha,\beta)}{B(\alpha,\beta)}$ and $\E_w[\ln(1-X)]=\psi(\beta)-\psi(\alpha+\beta)$.
\end{proof}
\end{comment} 

\subsection{Comparison to the Betting Strategy of \citet{Waudby-Smith--Wu--Ramdas--Karampatziakis--Mineiro2022}}
\label{app:sec:wswrkm}
As alluded to earlier, \citet{Waudby-Smith--Wu--Ramdas--Karampatziakis--Mineiro2022} proposed the general construction of a time-uniform lower confidence bound for non-negative random variables, based on the one-sided betting in disguise.
Beyond the meta strategy, they also suggested to use a certain betting strategy in their Eq.~(12), Eq.~(14), and Eq.~(15).
Concretely, for a hyperparameter $c\in[1/4,3/4]$, they proposed a betting scheme defined as\footnote{The original proposal considered a doubly robust estimator, and we simplify it by setting $k_t=0$ therein.}
\begin{align*}
b_t(\nu)\defeq \min\biggl\{\nu\sqrt{\frac{2\log\frac{1}{\delta}}{\hat{\sigma}_{t-1}^2t\log(t+1)}},c\biggr\}
\end{align*}
when $\nu>0$ is a candidate mean parameter, 
where 
\begin{align*}
\hat{\sigma}_t^2&\defeq \frac{1}{t+1}\biggl(\sigma_0^2+\sum_{i=1}^t(\ips_t^\pi-\bar{r}_t^\pi)^2\biggr)
\quad\text{and}\quad
\overline{r}_t^\pi\defeq \min\biggl\{\frac{1}{t}\sum_{i=1}^t \ips_i^\pi, 1\biggr\}.
\end{align*}
In the definition above, $\hat{\sigma}_t^2$ acts as a \emph{regularized} empirical variance,
where $\sigma_0^2$, which is a hyperparameter acting as a \emph{prior} on the variance, critically influences the performance in the small-to-moderate sample regime.
When $\sigma_0^2$ is too large (compared to the true variance), the amount of betting $b_t(\nu)$ will be very small and thus will not be able to sufficiently increase the log wealth, resulting in very loose confidence bounds.
When $\sigma_0^2$ is too small, the betting $b_t(\nu)$ is sensitive to the variability in the variance estimate and becomes unstable.
Consequently, the confidence bounds tend to have a larger variability as we have observed in the toy experiment in Section~\ref{sec:exp}. %

\paragraph{Simulation.} In the simulation setup in Appendix~\ref{app:sec:lbup}, we demonstrate the performance of this method; see WSWRKM in Figure~\ref{fig:lbup_relerr}. In this experiment, we set $c=1/2$ as suggested, and varied $\sigma_0^2\in\{10^{-4},1,10^4\}$ to demonstrate the effect of $\sigma_0^2$.

We first examine the role of $\sigma_0^2$. 
As alluded to above, an extremely small $\sigma_0^2$ in this case (i.e., $\sigma_0^2=10^{-4}$) starts off closely following the UP-LCB, but then dominated by a moderate $\sigma_0^2=1$ when $t\gtrsim 10^3$. If we set $\sigma_0^2$ extremely large (i.e., $\sigma_0^2=10^{4}$), it takes a significant amount of observations to result in a nonvacuous LCB. 
We note that this is a rather benign setting, since with Gamma random variables exhibit a light tail.
The behavior of WSWRKM under a heavy-tail setting was discussed in Section~\ref{sec:exp} in the main text.

Overall, the WSWRKM-LCB performs reasonably well, but it is outperformed by pCRP$^\star$ in the small-sample regime, and LBUP($r$)-LCB with $r\ge 2$ in the large-sample regime.
We also remark that \citet{Waudby-Smith--Wu--Ramdas--Karampatziakis--Mineiro2022} only proved its statistical validity, without establishing a finite-sample convergence rate of the LCB to the true mean.
Compared to the hyperparameter-free nature of UP-LCB, the presence of additional hyperparameters, $c$ and $\sigma_0^2$, in WSWRKM-LCB is also undesirable in practice.

\subsection{Proof of Theorem~\ref{thm:up_lcb_rate} (Convergence Rate Analysis for UP-LCB and pCRP\texorpdfstring{$^\star$}{*}-LCB)}

We restate Theorem~\ref{thm:up_lcb_rate} in two separate statements, and prove them separately. 
Technical lemmas are deferred to Appendix~\ref{app:sec:technical}.

\begin{theorem}[First part of Theorem~\ref{thm:up_lcb_rate}]
\label{thm:up_lcb_rate_1}
Let $n\ge 1$ and
define
$\blue{F_n^{(\dt)}} \defeq \ln\fr{\sqrt{\pi(n+1)}}{\dt^2}$.
Then, with probability $\ge 1-2\dt$,
\[
0
\le \mu - \lcbup{n}{\dt}(Y_{1:n})
\le \mu - \lcbpcrp{n}{\dt}(Y_{1:n})
\le \sqrt{\fr{48\sigma^2}{n} F_n^{(\dt)}} \vee \fr{12\mu}{n}F_n^{(\dt)}.
\]
\end{theorem}

Recall the definition of the smoothed variance
\[
\WW_b[Y]\defeq \EE\biggl[\fr{(Y-\EE[Y])^2}{1+\fr{b}{\EE[Y]}(Y-\EE[Y])}\biggr].
\]
\begin{theorem}[A Full Version of Second Part of Theorem~\ref{thm:up_lcb_rate}]
\label{thm:up_lcb_rate_2}
Pick any $\eps\in(0,\fr12]$.
Suppose that $(Y_t)_{t=1}^\infty$ is an independent identically distributed (i.i.d.), nonnegative random process, with $\mu\defeq\EE[Y_1]$ and $\sigma^2\defeq \VV[Y_1]$.
Let $b_n^{(\dt)}\defeq \sqrt{\fr{\mu^2}{2\sigma^2}\fr{F_n^{(\dt)}}{n}}$.
With probability $\ge 1-2\dt$, for any 
\[
n\ge \biggl(12\Bigl(1+\fr{4}{\eps}\Bigr) \vee 48\Bigl(1+\fr{4}{\eps}\Bigr)^2\fr{\mu^2}{\sigma^2}\biggr) F_n^{(\dt)},
\]
we have
\begin{align}
0\le \mu-\lcbup{n}{}(Y_{1:n})
&\le \mu-\lcbpcrp{n}{}(Y_{1:n})\nonumber\\
&\le 
\inf_{b\in(0,1-\eps]}
\biggl\{
\frac{b}{\mu} \WW_b[Y_1]
+\fr{\mu}{b}
\fr{F_n^{(\dt)}}{n}
\biggr\}\label{eq:refined_lcb1}\\
&\le 2\sqrt{
\frac{F_n^{(\dt)}}{n}
\WW_{b_n^{(\dt)}}[Y_1]
}.\label{eq:refined_lcb2}
\end{align}
\end{theorem}

\subsubsection{Proof of Theorem~\ref{thm:up_lcb_rate_1}}

Theorem~\ref{thm:up_lcb_rate_1} is an immediate consequence of Lemma~\ref{lem:lcb} and~\ref{lem:ucb} below.
\begin{lemma}\label{lem:lcb}
With probability $\ge 1-\dt$, $\mu\ge \lcbup{n}{\dt}(Y_{1:n}) \ge \lcbpcrp{n}{\dt}(Y_{1:n})$ for any $n\ge 1$.
\end{lemma}
\begin{proof}
Since $(\wealthup_t(Y_{1:t};\mu))_{t=1}^n$ is a nonnegative martingale,
by Ville's inequality, we have
\begin{align*}
\PP\Bigl(
\sup_{t\ge 1}\wealthup_t(Y_{1:t};\mu)\ge \frac{1}{\dt}
\Bigr)
\le \dt,
\end{align*}
which concludes the proof for the first inequality.
The second inequality is trivial by~\ref{eq:up_vs_pcrp}
\end{proof}

\begin{lemma}\label{lem:ucb}
Let
\begin{align*}
G_n^{(\dt)}\defeq \sqrt{\fr{12\sigma^2}{n} F_n^{(\dt)}} \vee \fr{6\mu}{n}F_n^{(\dt)}.
\end{align*}
With probability $\ge 1-\dt$,
\[
\mu
\le \lcbpcrp{n}{\dt}(Y_{1:n}) + 2G_n^{(\dt)} 
\le \lcbup{n}{\dt}(Y_{1:n}) + 2G_n^{(\dt)}
\]
for any $n\ge 1$. 
\end{lemma}

\begin{proof}
It suffices to prove the inequality for $\lcbpcrp{n}{\dt}(Y_{1:n})$ due to Eq.~\eqref{eq:up_vs_pcrp}.
We first note that, if $n < 12F_n^{(\dt)}(1\vee \frac{4\sigma^2}{\mu^2})$, 
we deterministically have $G_n^{(\dt)}> \fr{\mu}{2}$, which implies that
\begin{align*}
  \mu - \lcbpcrp{n}{\dt}(Y_{1:n})\le \mu < 2 G_n^{(\dt)},
\end{align*}
which proves the claim.

Hence, hereafter, we thus assume $n \ge 12F_n^{(\dt)}(1\vee \frac{4\sigma^2}{\mu^2})$ and show a slightly stronger bound
\begin{align}
\mu-\lcbpcrp{n}{\dt}(Y_{1:n}) \le G_n^{(\dt)}.
\label{eq:ub_claim}
\end{align}
In this regime, if we define
\begin{align*}
\nu_o\defeq\mu-G_n^{(\dt)},
\end{align*}
we have $\nu_o>0$, since $G_n^{(\dt)}\le\frac{\mu}{2}<\mu$.

Recall that $\nu\mapsto \wealthpcrp_n(Y_{1:n};\nu)$ is monotonically decreasing and $\lcbpcrp{n}{\dt}(Y_{1:n})$ is the unique root of $\wealthpcrp_n(Y_{1:n};\nu) =\frac{1}{\dt}$.
Therefore, to prove the desired claim in Eq.~\eqref{eq:ub_claim}, it suffices to show that 
\begin{align}
    \wealthpcrp_n(Y_{1:n};\nu_o) 
    >\frac{1}{\dt},
    \label{eq:ub_suff_cond}
\end{align}
since it implies that $\nu_o<\lcbpcrp{n}{\dt}(Y_{1:n})$.
By the definition of $\wealthpcrp_n$, it suffices to show that there exists $b^*\in(0,1)$ such that
\begin{align}
\label{eq:crp_wealth_larger_than_threshold}
\frac{1}{n}\ln\wealthcrp{b^*}_n(Y_{1:n};\nu_o) > \frac{1}{n}\ln\frac{\sqrt{\pi(n+1)}}{\dt}.
\end{align}
We will construct such $b^*$ below. 

Define
\begin{align*}
A\defeq \fr{\mu-\nu_o}{\nu_o} \quad\text{and}\quad 
B\defeq \fr{\sigma^2 + (\mu-\nu_o)^{2}}{\nu_o^2},
\end{align*}
and set
\[
b^*\defeq \fr{A}{2(A+2B)}\le \fr12.\footnote{The optimal choice of $b$ is $1-\sqrt{\frac{B}{A+B}}$, but the rate does not change in the current analysis.}
\]
By applying Lemma~\ref{lem:240903_concentration} with $\nu_o=\mu-G_n^{(\dt)}$ and $b^*$ chosen above, we have:
with probability $\ge 1-\dt$, for any $n\ge 1$, for any $\pi^*\in\Pi$,
\begin{align*}
\frac{1}{n}\ln\wealthcrp{b^*}_n(Y_{1:n};\nu_o)
&\ge
b^*\frac{\mu-\nu_o}{\nu_o} 
-\fr{(b^*)^2}{1-b^*} \fr{\sigma^2 + (\mu-\nu_o)^2}{\nu_o^2}
-\frac{1}{n}\ln\fr{1}{\dt}.\\
&=b^* A - \fr{(b^*)^2}{1-b^*} B - \frac{1}{n}\ln\fr{1}{\dt}
\\&= \fr{A^2}{2(A+2B)} - \frac{1}{n} \ln \fr{1}{\dt}.
\end{align*}
The last equality follows from the choice of $b^*$.

To show Eq.~\eqref{eq:crp_wealth_larger_than_threshold}, it remains to show that
\begin{align*}
\fr{A^2}{A+2B}\ge \frac{2F_n^{(\dt)}}{n}.
\end{align*}
We prove by contradiction:
if $\fr{A^2}{A+2B}< \frac{2F_n^{(\dt)}}{n}$, or equivalently
\begin{align}\label{eq:contradiction}
\frac{(\mu-\nu_o)^2}{2(\sigma^2 + (\mu-\nu_o)^2) + (\mu-\nu_o)\nu_o} < \frac{2F_n^{(\dt)}}{n},
\end{align}
then $G_n^{(\dt)}=\mu-\nu_o<G_n^{(\dt)}$.
We consider the following two cases separately.

  \paragraph{Case 1.} $\sigma^2 + (\mu-\nu_o)^2 \ge (\mu-\nu_o)\nu_o$.\\
  In this case, from Eq.~\eqref{eq:contradiction}, we have
\begin{align*}
\frac{2F_n^{(\dt)}}{n} > \frac{(\mu-\nu_o)^2}{3(\sigma^2 + (\mu-\nu_o)^2)},
\end{align*}
which implies that
  \begin{align*}
    (G_n^{(\dt)})^2=(\mu-\nu_o)^2
    < \frac{\frac{6F_n^{(\dt)}}{n}}{1-\frac{6F_n^{(\dt)}}{n}}\sigma^2 
    \le \frac{12\sigma^2 F_n^{(\dt)}}{n},
  \end{align*}
  which is a contradiction.
  Here, the last inequality follows from the assumption $n\ge 12F_n$.
  
  \paragraph{Case 2.} $\sigma^2 + (\mu-\nu_o)^2 < (\mu-\nu_o)\nu_o$.\\
  In this case, from Eq.~\eqref{eq:contradiction}, we have
\begin{align*}
\frac{2F_n^{(\dt)}}{n} > \frac{(\mu-\nu_o)^2}{3(\mu-\nu_o)\nu_o}=\frac{\mu-\nu_o}{3\nu_o},
\end{align*}
which implies that
\begin{align*}
G_n^{(\dt)}=\mu-\nu_o
< \frac{6\nu_o F_n^{(\dt)}}{n}
< \frac{6\mu F_n^{(\dt)}}{n},
\end{align*}
which is a contradiction.
Here, the last inequality follows since $\nu_o= \mu-G_n^{(\dt)}<\mu$.
This conclude the proof.
\end{proof}

\subsubsection{Proof of Theorem~\ref{thm:up_lcb_rate_2}}
Let $\hv\defeq \hv_{\mathsf{UP}}^{(\dt)}(Y_{1:n})$.
Note that 
$\wealthup_n(Y_{1:n};\hv)= \frac{1}{\dt}$ by the definition of UP-LCB.
Let $\blue{Z_1}\defeq \fr{Y_1}{\hv}$.
Since $\wealthup_n(Y_{1:n};\nu)\ge \wealthpcrp_n(Y_{1:n};\nu)$ from the regret guarantee of Cover's UP in Eq.~\eqref{eq:up_vs_pcrp}, by the definition of $\wealthpcrp_n(Y_{1:n};\nu)$, we have, for any $b\in(0,1)$, 
\begin{align*}
\fr{1}{n}\ln\fr{\sqrt{\pi(n+1)}}{\dt}
&\ge \fr{1}{n}\ln\wealthcrp{b}_n(Y_{1:n};\hv)\\
&=\fr{1}{n}\sum_{t=1}^n \ln\Bigl(
1-b+b\fr{Y_t}{\hv}
\Bigr)\\
&\ge b(\EE[Z_1]-1) - \EE\biggl[\fr{b^2(Z_1-1)^2}{1+b(Z_1-1)}\biggr] - \fr{1}{n}\ln\fr{1}{\dt},
\end{align*}
where the last inequality holds with probability $\ge 1-\dt$ by Lemma~\ref{lem:240903_concentration_new}.
We define $\blue{\Delta}\defeq \mu-\hv$, and we assume that $\Delta \ge 0$, which happens with probability $\ge 1-\dt$.
Note that $\fr{\Delta}{\hv} = \fr{\mu}{\hv}-1=\EE[Z_1]-1\ge 0$.
Rearranging the inequality, we then have
\begin{align}
\Delta
&\le \hv\biggl(
b\EE\biggl[\fr{(Z_1-1)^2}{1+b(Z_1-1)}\biggr] + \fr{1}{b} \fr{F_n^{(\dt)}}{n}
\biggr).
\label{eq:intermed2}
\end{align}
We bound the first term as follows:
\begin{align*}
\EE\biggl[\fr{(Z_1-1)^2}{1+b(Z_1-1)}\biggr]
&\le 2\EE\biggl[\fr{(Z_1-\EE[Z_1])^2 + (\EE[Z_1]-1)^2}{1+b(Z_1-1)}\biggr]
\\
&\le 2\EE\biggl[\fr{(Z_1-\EE[Z_1])^2 }{1+b(Z_1-\EE[Z_1])}\biggr]
+ \fr{2\Delta^2}{\hv^2}
\EE\biggl[\fr{1}{1+b(Z_1-1)}\biggr]
\tag{$\because \EE[Z_1]\ge 1$}
\\
&\sr{(a)}{\le} 2\EE\biggl[\fr{(Z_1-\EE[Z_1])^2 }{1+b(Z_1-\EE[Z_1])}\biggr]
+\fr{\Delta}{2\hv}.
\end{align*}
Here, we show that $(a)$ is true given $n\ge (\fr{12}{c} \vee \fr{48\sigma^2}{c^2\mu^2}) F_n^{(\dt)}$ for $c=\fr{\eps}{\eps+4}$ and $b\in(0,1-\eps]$.
To see this, Theorem~\ref{thm:up_lcb_rate} along with the requirement on $n$ ensures 
\[
0\le \Delta = \mu-\hv \le c\mu ~,
\]
which is equivalent to
\begin{align}
\fr{4}{\eps+4}\mu
=(1-c)\mu 
\le \hv \le \mu
\label{eq:intermed3}
\end{align}
or
\[
0\le \Delta \le \fr{c}{1-c}\hv = \fr{\eps}{4}\hv.
\]
This leads to, using $Z_1 \ge 0$ and $b \in (0,1-\eps]$,
\begin{align*}
\fr{2\Delta^2}{\hv^2}
\EE\biggl[\fr{1}{1+b(Z_1-1)}\biggr]
\le \fr{2\Delta^2}{\hv^2}
\fr{1}{1-b} 
\le \fr{2\Delta}{\hv^2}\fr{\Delta}{\eps}
\le \fr{2\Delta}{\hv^2}\fr{\hv}{4}
= \fr{\Delta}{2\hv}~,
\end{align*}
concluding the proof of $(a)$ above.

We now apply the upper bound of $\EE\Bigl[\fr{(Z_1-1)^2}{1+b(Z_1-1)}\Bigr]$ above to Eq.~\eqref{eq:intermed2} and solve it for $\Dt$ to obtain
\begin{align}
\Delta 
\le (2-b)\Delta
&\le 
\frac{4b}{\hv}\EE\biggl[\fr{(Y_1-\mu)^2}{1+\fr{b}{\hv}(Y_1-\mu)}\biggr]
+\fr{\hv}{b}\fr{F_n^{(\dt)}}{n}
\defeq h\Bigl(\fr{b}{\hv}\Bigr),
\end{align}
where $h(q)\defeq 4q\EE\Bigl[\fr{(Y_1-\mu)^2}{1+q(Y_1-\mu)}\Bigr]
    +\fr1q\fr{F_n^{(\dt)}}{n}$.
Taking infimum over $b\in(0,1-\eps]$,
\begin{align*}
\Delta
\le 
\inf_{b\in[0,1-\eps)}
h\Bigl(\fr{b}{\hv}\Bigr)
=\inf_{q\in[0,\fr{1-\eps}{\hv})} h(q)
\le \inf_{q\in[0,\fr{1-\eps}{\mu})} h(q)
= \inf_{b\in[0,{1-\eps})} h\Bigl(\fr{b}{\mu}\Bigr).
\end{align*}
The second inequality holds since we assume $\mu\ge \hv$.
This concludes the proof for the first inequality in Eq.~\eqref{eq:refined_lcb1}.

To prove the second inequality in Eq.~\eqref{eq:refined_lcb2}, we rewrite the inequality in Eq.~\eqref{eq:refined_lcb1} as
\begin{align}
\Delta
&\le 
\inf_{b\in(0,1-\eps]}
\{
f(b) + g(b)
\}
\le \inf_{b\in(0,\fr14]}
\{
f(b) + g(b)
\},
\label{eq:intermed4}
\end{align}
where $f(b)\defeq \frac{b}{\mu}\EE\Bigl[\fr{(Y_1-\mu)^2}{1+\fr{b}{\mu}(Y_1-\mu)}\Bigr]$
and $g(b)\defeq \fr{\mu}{b}
\fr{F_n^{(\dt)}}{n}$.
Note that $f(b)$ is monotonically increasing and $g(b)$ is monotonically decreasing over $b\in [0,1]$.
We now show that $f(\fr14) \ge g(\fr14)$, which implies that $f(b_o)=g(b_o)$ for some $0<b_o\le \fr14$.
To show this, note that 
\begin{align*}
f\Bigl(\fr14\Bigr)
&=  \EE\biggl[\fr{(Y_1-\mu)^2}{4\mu+(Y_1-\mu)}\biggr]\\
&= \EE\biggl[\fr{(Y_1-\mu)^2}{3\mu+Y_1}\biggr]\\
&\ge \EE\biggl[\fr{(Y_1-\mu)^2}{2Y_1}\onec{Y_1\ge 3\mu}\biggr]\\
&\ge \EE\biggl[\fr{\fr{Y_1^2}{2}-\mu^2}{2Y_1}\onec{Y_1\ge 3\mu}\biggr]\tag{$\because (a-b)^2\ge \fr12 c^2-b^2$}\\
&\ge \EE\biggl[\fr{\fr{Y_1^2}{2}-\fr{Y_1^2}{9}}{2Y_1}\onec{Y_1\ge 3\mu}\biggr]\\
&\ge\EE\biggl[\fr{7}{36}Y_1\onec{Y_1\ge 3\mu}\biggr]\\
&\ge \fr{7}{12}\mu\\
&\ge 4\mu \fr{F_n^{(\dt)}}{n}
= g\Bigl(\fr14\Bigr).
\end{align*}
Here, the last inequality follows from the assumption that $n\ge 7F_n^{(\dt)}$.
Hence, if we plug in the root $b_o$ to Eq.~\eqref{eq:intermed4}, then we have
\begin{align*}
\Delta \le f(b_o)+g(b_o) = 2\sqrt{\fr{F_n^{(\dt)}}{n} \EE\biggl[\fr{(Y_1-\mu)^2}{1+\fr{b_o}{\mu}(Y_1-\mu)}\biggr]}.
\end{align*}
To further upper bound this term, it suffices to find a deterministic lower bound on $b_o$, since, by Lemma~\ref{lem:basic} stated below, if $0\le b_\ell\le b_o$,
\begin{align*}
\EE\biggl[\fr{(Y_1-\mu)^2}{1+\fr{b_o}{\mu}(Y_1-\mu)}\biggr]
\le 2\EE\biggl[\fr{(Y_1-\mu)^2}{1+\fr{2b_\ell}{\mu}(Y_1-\mu)}\biggr].
\end{align*}
To find such a lower bound $b_\ell$, 
we note that, if we define $\eta(b)\defeq \fr{2\sigma^2}{\mu}b \ge f(b)$  
and $b\mapsto \eta(b)$ is monotonically increasing, and thus the root $b_o'$ of the equation $\eta(b)=g(b)$ must be smaller than $b_o$.
Hence, solving $\eta(b)=\fr{2\sigma^2}{\mu}b=\fr{\mu}{b}\fr{F_n^{(\dt)}}{n}=g(b)$ yields the root
\begin{align*}
b_o'=b_n^{(\dt)}\defeq \sqrt{\fr{\mu^2}{2\sigma^2}\fr{F_n^{(\dt)}}{n}}.
\end{align*}
Note that we require $n>\fr{\mu^2}{\sigma^2}F_n^{(\dt)}$ to ensure that the root $b_n^{(\dt)}$ lies in $(0,\fr12)$, which is assumed in the statement.
Finally, 
we have $\Delta\le f(b_o)+g(b_o) \le f(b_n^{(\dt)}) + g(b_n^{(\dt)})$, which concludes the proof.
\jmlrQED

\subsubsection{Technical Lemmas}
\label{app:sec:technical}

Here, we state and prove technical lemmas used in the proofs above.
Note that we obtain time-uniform guarantees below immediately by applying Ville's inequality in place of Markov's inequality.
\begin{lemma}
\label{lem:240903_concentration}
Let $Y_1,\ldots,Y_t$ be i.i.d. nonnegative random variables.
For any ``betting'' $b\in[0,1]$ and a ``reference'' mean $\nu>0$, we have
\begin{align*}
\PP\del[2]{
\frac{1}{n}\sum_{t=1}^n \ln\del[2]{1-b+b\frac{Y_t}{\nu}}
\ge 
b\frac{\mu-\nu}{\nu} 
-\fr{b^2}{1-b} \fr{\VV[Y_1] + (\mu-\nu)^2}{\nu^2}
-\frac{1}{n}\ln\frac{1}{\dt}
}\ge 1-\dt.
\end{align*}
\end{lemma}
\begin{proof}
Applying Lemma~\ref{lem:240903_var}  to Lemma~\ref{lem:240903_concentration_new} concludes the proof.
\end{proof}

\begin{lemma}
  \label{lem:240903_concentration_new}
  Let $Y_1,\ldots,Y_t$ be i.i.d. nonnegative random variables.
  For any ``betting'' $b\in[0,1]$ and a ``reference'' mean $\nu>0$, we have
  \begin{align*}
    \PP\del[2]{
      \frac{1}{n}\sum_{t=1}^n \ln\del[2]{1-b+b\frac{Y_t}{\nu}}
      \ge 
      b\frac{\mu-\nu}{\nu} 
      -\EE\sbr[2]{\fr{b^2 \fr{(Y_1 - \nu)^2}{\nu^2}}{1 + b\fr{Y_1 - \nu}{\nu}}} 
      -\frac{1}{n}\ln\frac{1}{\dt}
    }\ge 1-\dt.
  \end{align*}
\end{lemma}
\begin{proof}
  Use Lemma~\ref{lem:240903_concentration_basic} with $Z_t\gets b\frac{Y_t-\nu}{\nu}$.
\end{proof}

\begin{lemma}
\label{lem:240903_concentration_basic}
Let $Z_1,\ldots,Z_t$ be i.i.d. random variables supported over $(-1,\infty)$.
Then, we have
\begin{align*}
\PP\del[2]{
-\frac{1}{n}\sum_{t=1}^n \ln\del[1]{1+Z_t}
+\EE[Z_1] \le \EE\sbr[2]{\fr{Z_1^2}{1 + Z_1}} + \frac{1}{n}\ln\frac{1}{\dt}
}\ge 1-\dt.
\end{align*}
\end{lemma}
\begin{proof}
Note that the following is a nonnegative random variable.
\begin{align*}
M_n = \prod_{t=1}^n \fr{\fr{1}{1 + Z_t}}{\EE[\fr{1}{1 + Z_t}]} ~.
\end{align*}
Thus, by Markov's inequality $\PP(\ln\frac{M_n}{\EE[M_n]}\ge \ln \frac{1}{\dt})\le \dt$, or equivalently, w.p. at least $1-\dt$, we have
\begin{align*}
-\frac{1}{n}\sum_{t=1}^n \ln(1 + Z_t)  - \frac{1}{n}\ln\frac{1}{\dt}
&< \ln\EE\biggl[\fr{1}{1 + Z_1}\biggr]
\\&\le \EE\biggl[\fr{1}{1 + Z_1}\biggr] - 1 && (\because \ln(x)\le x-1)
\\&=   \EE\biggl[\fr{-Z_1}{1 + Z_1}\biggr]
\\&=   \EE\biggl[\fr{Z_1^2}{1 + Z_1}\biggr]  - \EE[Z_1].
\end{align*}
The last equality holds since $-\frac{t}{1+t}=\frac{t^2}{1+t}-t$.
\end{proof}

\begin{lemma}\label{lem:240903_var}
We have
\begin{align*}
\EE\sbr[3]{\fr{b^2 \fr{(Y_1 - \nu)^2}{\nu^2}}{1 + b\fr{Y_1 - \nu}{\nu}}}
\le \fr{b^2}{1-b} \fr{\VV[Y_1] + (\mu-\nu)^2}{\nu^2}.
\end{align*}
\end{lemma}
\begin{proof}
Consider
\begin{align*}
\fr{(Y_1 - \nu)^2}{1 + b\fr{Y_1 - \nu}{\nu}} 
= \fr{\nu{(Y_1-\nu)^2}}{b Y_1 + (1-b)\nu} 
\le \fr{{(Y_1-\nu)^2} }{(1-b)}. 
\end{align*}
Taking the expectation, we have $\EE[(Y_1-\nu)^2] = \EE[(Y_1-\mu+\mu-\nu)^2] = \VV[Y_1]+(\mu-\nu)^2$, which concludes the proof.
\end{proof}

\begin{lemma}
\label{lem:basic}
For any $y\ge 0$, $0\le b'\le b\le \fr12$, we have
\begin{align*}
\fr{1}{1+b\fr{y-\mu}{\mu}}
\le \fr{2}{1+2b'\fr{y-\mu}{\mu}}.
\end{align*}
\end{lemma}
\begin{proof}
Note that the denominators in both sides are positive.
Hence, the inequality is equivalent to
\begin{align*}
1+2b'\fr{y-\mu}{\mu} &\le 2+2b\fr{y-\mu}{\mu}
\Leftrightarrow (b-b')\Bigl(1-\fr{y}{\mu}\Bigr) \le 1.
\end{align*}
The last inequality readily follows from the assumptions $0\le b'\le b\le \fr12$ and $y\ge 0$.
\end{proof}

\subsection{Empirical-Bernstein-Type Relaxation of UP-LCB}
\label{app:sec:emp_bern_relaxation}
As alluded to earlier in Section~\ref{sec:finite_sample}, here we provide an empirical-Bernstein-type relaxation of pCRP$^\star$-LCB.
\begin{theorem}[Empirical-Bernstein-type relaxation of \pcrplcb{}]
\label{thm:emp_bern_relaxation}
Let $\hat{\mean}_n \defeq \frac{1}{n}\sum_{t=1}^n Y_t$ and $\hat{\mathsf{V}}_n\defeq \frac{1}{n}\sum_{t=1}^n (Y_t-\hat{\mean}_n)^2$ denote the empirical mean and variance, respectively.
Let $H_{n}^{(\dt)}\defeq\ln\frac{\sqrt{\pi(n+1)}}{\dt}$ and let $\lcbeb{n}{\dt}(Y_{1:n})
\defeq \hat{\mean}_n-\Delta_n^{(\dt)}$, where
\begin{align*}
\Delta_n^{(\dt)}
\defeq \frac{1}{1-\frac{2}{n}H_{n}^{(\dt)}} 
\Biggl(
\frac{\hat{\mean}_n}{n} H_{n}^{(\dt)}
+\sqrt{\frac{\hat{\mean}_n^2}{n^2} (H_{n}^{(\dt)})^2
+ \frac{4\hat{\mathsf{V}}_n}{n} H_{n}^{(\dt)} 
\Bigl(1-\frac{2}{n}H_{n}^{(\dt)}\Bigr) 
}
\Biggr).
\end{align*}
Under the same setting of Proposition~\ref{prop:lcb}, with probability at least $1-\dt$, for all $n\ge 1$ such that $H_{n}^{(\dt)}<\fr12$, we have $\mu\ge \lcbeb{n}{\dt}(Y_{1:n})$.
\end{theorem}

\begin{proof}
By Ville's inequality, with probability $1-\dt$, we have, for any $n\ge 1$,
\begin{align*}
\ln\frac{1}{\dt} {\ge} \ln\wealthup_t(Y_{1:n};\nu)
\stackrel{(a)}{\ge}
\ln\wealthpcrp(Y_{1:n};\nu)
={\sup_{b\in[0,1]}\ln\wealthcrp{b}(Y_{1:n};\nu)}
-\ln{\sqrt{\pi(n+1)}},
\end{align*}
which is equivalent to
\begin{align*}
\frac{1}{n}\sup_{b\in[0,1]} \sum_{t=1}^n \ln\Bigl(1-b+b\frac{Y_t}{\nu}\Bigr) 
\le \frac{1}{n} H_{n}^{(\dt)}.
\end{align*}
Here, $(a)$ follows from Eq.~\eqref{eq:up_vs_pcrp}.

Now, we apply Lemma~\ref{lem:generalized_lower_bound} for $n=1$ and obtain
\begin{align*}
\ln(1-b+bZ)
&\ge b((1-Z)^2-(1-Z)) + (1-Z)^2\ln(1-b)\\
&= b(Z^2-Z)+(Z^2-2Z+1)\ln(1-b),
\end{align*}
which holds for any $b\in[0,1)$ and $Z>0$.
Applying this inequality to each summand, we have
\begin{align*}
\fr{H_{n}^{(\dt)}}{n}
&\ge
\frac{1}{n}\sup_{b\in[0,1]} \sum_{t=1}^n \ln\Bigl(1-b+b\frac{Y_t}{\nu}\Bigr)\\
&\ge \frac{1}{\nu^2} \sup_{b\in[0,1]}\Bigl\{
((\hat{\mathsf{V}}_n+\hat{\mean}_n^2)-\hat{\mean}_n\nu)b
+((\hat{\mathsf{V}}_n+\hat{\mean}_n^2)-2\hat{\mean}_n\nu+\nu^2)\ln(1-b)\Bigr\}.\\
&= \frac{1}{\nu^2} \sup_{b\in[0,1]}\Bigl\{
B b
+(B-A)\ln(1-b)\Bigr\}\\
&\stackrel{(b)}{\ge} \frac{1}{\nu^2} \sup_{b\in[0,1]}\Bigl\{
B b
+(B-A)\frac{-b}{1-b}\Bigr\}\\
&\stackrel{(c)}{\ge} \frac{1}{\nu^2} \frac{A^2}{2(2B-A)},
\end{align*}
where $A\defeq (\hat{\mean}_n-\nu)\nu$ and $B\defeq (\hat{\mathsf{V}}_n+\hat{\mean}_n^2)-\hat{\mean}_n\nu$.
Note that $(b)$ follows from the elementary inequality $\ln(1-b) \ge \frac{-b}{1-b}$ for $b<1$, and $(c)$ follows by setting $b=\frac{A}{2B}$ to derive a lower bound.
We now wish to solve the equation 
\[
\fr{H_{n}^{(\dt)}}{n}\nu^2 = \frac{A^2}{2(2B-A)}
\]
with respect to $\nu$, which becomes equivalent to 
\[
\Bigl(1-2\fr{H_{n}^{(\dt)}}{n}\Bigr)x^2 - 2\fr{H_{n}^{(\dt)}}{n}\hat{\mean}_nx - 4\fr{H_{n}^{(\dt)}}{n}\hat{\mathsf{V}}_n=0,
\]
if we let $x\defeq \hat{\mean}_n-\nu$.
It is easy to check that $x=\hat{\mean}_n-\lcbeb{n}{\dt}(Y_{1:n})$ is the solution to this quadratic equation and thus a valid lower bound for $\mean$.
\end{proof}

\subsection{Proof for Theorem~\ref{thm:main_selection} (Regret Analysis for PUB)}

We provide a proof for $\hpi=\selectupsimple$, and the other case follows immediately by the same logic.
It suffices to show the second inequality.
Letting $2G_{n}^{(\dt)}[Y_1]$ denote the upper bound in Theorem~\ref{thm:up_lcb_rate},
we apply Theorem~\ref{thm:up_lcb_rate} to the process $\ips_{1:n}^\pi$ for each $\pi\in\Pi$ and take a union bound with $\dt\gets \dt'=\frac{\dt}{|\Pi|}$.
Under the good event with probability $\ge 1-2\dt$,
we have
\begin{align*}
\mean(\pi^*) - \mean(\selectupsimple) 
\stackrel{(a)}{\le} \mean(\pi^*) - \lcbup{n}{}(\ips_{1:n}^{\selectupsimple})  
\stackrel{(b)}{\le} \mean(\pi^*) - \lcbup{n}{}(\ips_{1:n}^{\pi^*})
\stackrel{(c)}{\le} 2G_n^{(\dt')}[\ips_1^{\pi^*}].
\end{align*}
Here, $(a)$ follows since $\mean(\ips^{\selectupsimple})\ge \lcbup{n}{}(\ips_{1:n}^{\selectupsimple})$, $(b)$ from the definition of the selection method in Eq.~\eqref{eq:up_select},
and $(c)$ from the upper bound of Theorem~\ref{thm:up_lcb_rate}.

The second part of the statement follows from the second part of Theorem~\ref{thm:up_lcb_rate} in place of the first part.
\jmlrQED

\subsection{Off-Policy Evaluation with Betting}
\label{sec:evaluation}
We can immediately apply the UP-LCB and pCRP$^\star$-LCB for off-policy evaluation as well.
Similar to \citet{Waudby-Smith--Wu--Ramdas--Karampatziakis--Mineiro2022}, we can construct the upper confidence bound (UCB) of the value of a policy using our LCB machinery, since 
\[
\breve{r}_t^\pi\defeq w_t^\pi (1-r_t) = w_t^\pi - \ips_t^\pi 
\]
is also a nonnegative random process.
Using $\EE[\breve{r}_t^\pi] = 1-\mean(\pi)$, we can construct the LCB from $\breve{r}_{1:n}^\pi$, from which we can construct the UCB of $\mean(\pi)$.
More precisely, we have:
\begin{proposition}
Pick any policy $\pi$.
With probability $\ge 1-2\dt$,
\begin{align*}
\lcbpcrp{n}{\dt}(\ips_{1:n}^\pi)
\le \lcbup{n}{\dt}(\ips_{1:n}^\pi) 
\le \mean(\pi)
\le 1-\lcbup{n}{\dt}({\breve{r}}_{1:n}^\pi)
\le 1-\lcbpcrp{n}{\dt}({\breve{r}}_{1:n}^\pi).
\end{align*}
\end{proposition}
Unlike \citet{sakhi24logarithmic}, our guarantee provides a direct control over the width of the confidence bounds.
The following guarantee is immediate from Theorem~\ref{thm:up_lcb_rate}:
\begin{theorem}[Evaluation]
\label{thm:evaluation} 
Pick any policy $\pi$.
Let $\breve{\var}(\pi)\defeq \VV[\breve{r}_1^\pi]$.
With probability $\ge 1-4\dt$, 
\begin{align*}
-\del[3]{\sqrt{\fr{48\breve{\var}(\pi)}{n} F_n^{(\dt)}} \vee \fr{12(1-\mean(\pi))}{n}F_n^{(\dt)} }
&\le \mean(\pi) - (1-\lcbpcrp{n}{\dt}({\breve{r}}_{1:n}^\pi))\\
&\le \mean(\pi) - (1-\lcbup{n}{\dt}({\breve{r}}_{1:n}^\pi))\\
&\le 0\\
&\le \mean(\pi) - \lcbup{n}{\dt}(\ips_{1:n}^\pi)\\
&\le \mean(\pi) - \lcbpcrp{n}{\dt}(\ips_{1:n}^\pi)
\le \sqrt{\fr{48\tilde{\var}(\pi)}{n} F_n^{(\dt)}} \vee \fr{12\mean(\pi)}{n}F_n^{(\dt)}.
\end{align*}
\end{theorem}

\section{Deferred Proofs for Off-Policy Learning}
\label{app:sec:proofs_learning}

\subsection{Proof of Proposition~\ref{prop:phi} (Examples of Score Functions)}

\begin{proof}
  Logarithmic smoothing is trivial by definition.
  For freezing, 
  the upper bound side is obvious.
  For the lower bound, we need to find $c_1$ and $c_2$ such that
  \begin{align*}
    f(x) \defeq\fr{\exp(-\phi(x)) - 1 + x}{x^2} \le \fr{1}{c_1 + c_2 x} 
  \end{align*}
  For this, if $x\le 1$ then we have $f(x) = \fr{1}{1+x}$.
  If $x > 1$, then we have $f(x) = 1/x$.
  Thus, using $\onec{x > 1} \le \fr{x}{1+x}$,
  \begin{align*}
    f(x) 
    &\le \onec{x\le 1} \fr{1}{1+x} + \onec{x > 1}  \fr{1}{x}
    \\&\le \onec{x\le 1} \fr{1}{1+x} + \fr{2}{1 + x} 
    \\&\le \fr{2}{1+x}.
  \end{align*}
  Thus, we have $c_1 = c_2 = \fr12$.
  For clipping, similar to freezing, if $x\le1$ then we have $f(x) = \fr{1}{1+x}$.
  If $x>1$, then we have $f(x) = \fr{-\fr12 + x}{x^2} \le \fr 1 x$.
  We can then proceed the identical derivation to Freezing to obtain $c_1 = c_2 = \fr12$.
\end{proof}

\subsection{Proof of Theorem~\ref{thm:learning-new} (Regret Analysis for Learning Algorithm)}

To derive the desired regret bound for our general estimator $\hpi_n \defeq \arg\max_{\pi\in\Pi} \sum_{t=1}^n \phi(\beta \ips_t^\pi)$, we consider the following two martingales:
\begin{align*}
  \text{(Upper deviation):}\quad
  U_n^\pi &\defeq
  \  \prod_{t=1}^n \frac{e^{\phi(\beta \ips_t^\pi)}}{\EE[e^{\phi (\beta \ips_t^\pi)}]}, \\
  \text{(Lower deviation):}\quad
  L_n^\pi &\defeq
  \  \prod_{t=1}^n \frac{e^{-\phi (\beta \ips_t^\pi)}}{\EE[e^{-\phi (\beta \ips_t^\pi)}]}.
\end{align*}
Throughout the proof we omit the subscript $t$ from $\ips_t^\pi$ inside the expectation, and use $\ips^\pi$ for simplicity.

By applying Ville's inequality~\citep{ville39etude} and taking the union bound over $\pi\in\Pi$, we have: with probability at least $1-2\dt$,
$U_n^\pi\le \frac{1}{\dt}$ and $L_n^\pi\le \frac{1}{\dt}$
for all $\pi \in \Pi$. 
Given this good event,
we have
\begin{align}
  -\ln\left(\EE[e^{-\phi(\beta \ips^{\pi^\star})}] \right) - \frac{1}{n}\ln\frac{|\Pi|}{\dt} 
  &\stackrel{(a)}{\le} \frac{1}{n} \sum_{t=1}^n \phi(\beta \ips_t^{\pi^\star}) \nonumber\\ 
  & \stackrel{(b)}{\le} \frac{1}{n} \sum_{t=1}^n \phi(\beta \ips_t^{\hpi_n}) \nonumber\\
  & \stackrel{(c)}{\le} \ln\left(\EE[e^{\phi(\beta \ips^{\hpi_n})}] \right) + \frac{1}{n}\ln\frac{|\Pi|}{\dt},
  \label{eq:ineq_chain_kj}
\end{align}
where $(a)$ follows from $L_n^{\pi^\star}\le \frac{1}{\dt}$,
$(b)$ follows by the definition of $\hpi_n$, and $(c)$ follows from $U_n^{\hpi_n}\le \frac{1}{\dt}$. 

We now further upper- and lower-bound this inequality.
Note that
\begin{align*}
  \ln(\EE[e^{\phi(\beta \ips^{\hpi_n})}])
  = \beta\EE[\ips^{\hpi_n}] + \ln (\EE[e^{\phi(\beta \ips^{\hpi_n}) - \EE  \beta \ips^{\hpi_n}}])
\end{align*}
Thus,
\begin{align}
  \fr1n\sum_t \fr1\beta \phi(\ips_t^{\hpi_n}) - \EE[\ips^{\hpi_n}] 
  &\le \underbrace{\fr1{\beta} \ln\fr{1}{\EE[e^{\phi(\beta \ips^{\hpi_n}) - \EE  \beta \ips^{\hpi_n}}]}}_{= -F_\beta(\phi)} + \fr{1}{n\beta} \ln(1/\dt) \label{eq:log_expected_inequality_ub_kj}
\end{align}
Note that $F_\beta(\phi) \ge 0$ by $\phi(x) \le \ln(1+x)$.

For the lower bound, we have
Note that
\begin{align*}
  \ln(\EE[e^{-\phi(\beta \ips^{\pi^*})}])
  &\le \EE[e^{-\phi(\beta \ips^{\pi^*})}] - 1
  \\&\le - \EE[\beta \ips^{\pi^*}] + \EE\biggl[\fr{\beta^2(\ips^{\pi^*})^2}{c_1 + c_2 \beta \ips^{\pi^*}} \biggr].
\end{align*}
Therefore,
\begin{align}
  \frac{1}{n}\sum_t -\fr{1}{\beta} \phi(\beta \ips_t^{\pi^*} ) + \EE[ \ips^{\pi^*}]\le \beta \EE\biggl[\fr{(\ips^{\pi^*})^2}{c_1 + c_2 \beta \ips^{\pi^*}}\biggr] +  \fr{1}{n\beta} \ln(1/\dt)~.
  \label{eq:log_expected_inequality_lb_kj}
\end{align}

By combining Eq.~\eqref{eq:log_expected_inequality_ub_kj} and Eq.~\eqref{eq:log_expected_inequality_lb_kj} through Eq.~\eqref{eq:ineq_chain_kj}, we have
\begin{align*}
  v(\pi^\star)-v(\hpi_n)
  &=\EE[\ips^{\pi^\star}] - \EE[\ips^{\hpi_n} ] \\
  &\le \beta \EE\biggl[\fr{(\ips^{\pi^*})^2}{c_1 + c_2 \beta \ips^{\pi^*}}\biggr] + F_\beta(\phi) + \fr{2}{\beta n} \ln\fr{|\Pi| }{\dt},
\end{align*}
which proves the desired claim.
\jmlrQED

\section{On Experiments and Additional Results}

\subsection{On the Heavy-Tail Setup in Section~\ref{sec:counter-example}}
\label{app:sec:counter-example}
We first provide a formal statement and its proof on the nonexistence of the fourth moment in the setup of Section~\ref{sec:counter-example}.

\begin{proposition}
\label{prop:counter_example}
Consider the discrete context space $\cX =\mathbb{N}= \{1,2,\ldots\}$ and a discrete action space $\mathcal{A}=\{1,\ldots,K\}$,
 where the context probability $p(x)$ is assigned such that
$p(x=i) \propto \fr{1}{i^2}$.
Suppose that $p(r=1|x=i,a=1)\ge \tau>0$, and a behavior policy is defined as $\piref(a=1|x)\defeq \frac{1}{x^\beta}$ for $\beta\ge 1$. 
If $\pi(a=1|x)\ge c$ for any $x\in\mathcal{X}$ for some $c>0$,
then the fourth moment $\EE[(\ips_t^\pi)^4]$ does not exist.    
\end{proposition}

\begin{proof}
Consider
\begin{align*}
\EE[(\ips_t^\pi)^4]
&= \EE_{p(x)\piref(a|x)p(r|a,x)}\biggl[ \biggl(\fr{\pi(A|X)}{\piref(A|X)}\biggr)^4 R^4\biggr]
\\&= \EE_{p(x)}\biggl[ \sum_{a\in\{0,1\}} \fr{\pi(a\mid X)^4}{\piref(a\mid X)^3} \E_{p(r|a,X)}[R^4]\biggr]
\\&\ge \tau\EE_{p(x)}\biggl[ \fr{\pi(a=1\mid X)^4}{\piref(a=1\mid X)^3}
\biggr]
\\&\gsim \sum_{i=1}^\infty \fr{1}{i^2} \fr{\pi(a = 1\mid x=i)^4}{\piref(a=1\mid x=i)^3} 
\\&\gsim \sum_{i=1}^\infty \fr{1}{i^2} \frac{1}{(1/i^\beta)^3}
\\&\gsim \sum_{i=1}^\infty {i^{3\beta-2}} 
=\infty.
\end{align*}
This concludes the proof.
\end{proof}

In Figure~\ref{fig:counter_example_realizations}, we present some realizations of the trajectories and LCBs that correspond to the summarization in Figure~\ref{fig:unbdd_up_eb_main}.

\begin{figure}[p]
    \centering
    \subfigure[]{%
        \includegraphics[width=.8\linewidth]{figs/unbdd_up_eb_wswrkm.pdf}
    }\\
    \subfigure[]{%
        \includegraphics[width=0.4\textwidth]{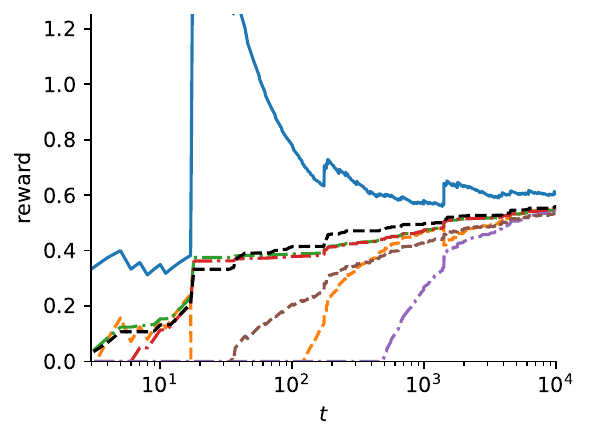}%
        \label{fig:subfig1}%
    }
    \subfigure[]{%
        \includegraphics[width=0.4\textwidth]{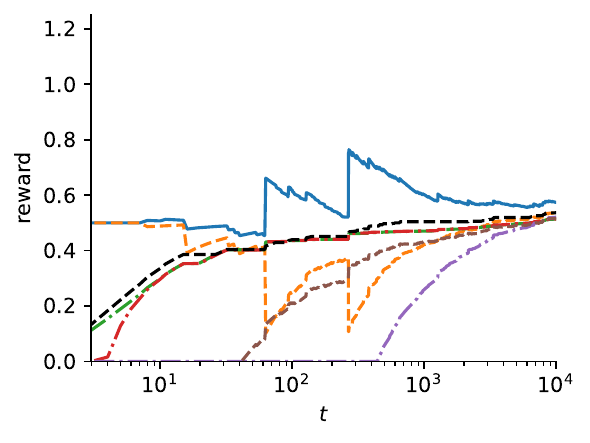}%
        \label{fig:subfig2}%
    }\\
    \subfigure[]{%
        \includegraphics[width=0.4\textwidth]{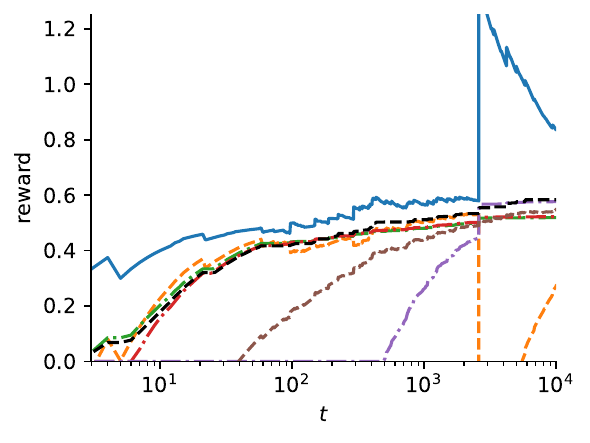}%
        \label{fig:subfig3}%
    }
    \subfigure[]{%
        \includegraphics[width=0.4\textwidth]{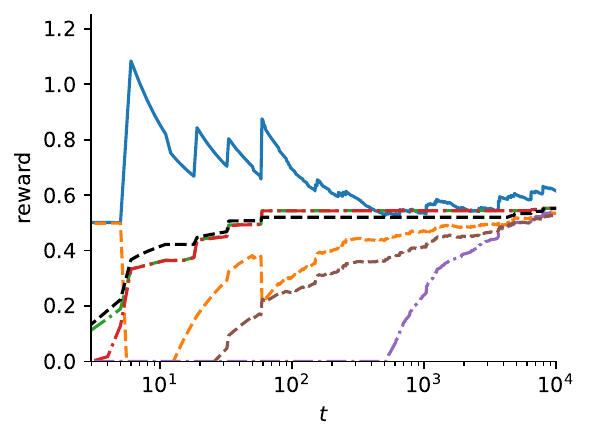}%
        \label{fig:subfig4}%
    }
    \caption{Comparison of the UP-based LCB with baseline LCBs. %
    The average behavior over $N=100$ random trials is presented in (a), and (b)-(d) show some realizations of the random runs. 
    These instances clearly demonstrate the failure cases of empirical-Bernstein-type bounds, which rely on the concentration of the empirical variance.
    }    
    \label{fig:counter_example_realizations}
\end{figure}

\subsection{On OP Learning and Selection Datasets}
Table~\ref{tab:datasets} summarizes the dimensions of each dataset.\vspace{-.5em}

\begin{table}[!ht]
\centering
\caption{Summary statistics of the datasets.}\vspace{.5em}
\begin{tabular}{c c c c}
\toprule
Dataset & PenDigits & SatImage & JPVowel \\
\midrule
\# features & 16 & 36 & 14 \\
\# classes & 10 & 6 & 9 \\
\bottomrule
\end{tabular}
\label{tab:datasets}
\end{table}\vspace{-.5em}

\subsection{OP Learning Baselines}
\label{app:sec:estimators}
The estimators we tested in the OP learning experiment are defined as follows:
\begin{align*}
\hat{\pi}_{\mathsf{PL}}&\defeq \arg\max_{\pi\in\Pi}  \sum_{t=1}^n \Bigl(\ips_t^\pi - {\beta}\sum_{a\in\mathcal{A}} \frac{\pi(a|x_t)}{\piref(a|x_t)}\Bigr),\\
\hat{\pi}_{\mathsf{ClippedIW}}&\defeq \arg\max_{\pi\in\Pi}  \sum_{t=1}^n \frac{\pi(a_t|x_t)}{\piref(a_t|x_t) \vee \beta} r_t,\\
\hat{\pi}_{\mathsf{IX}}&\defeq \arg\max_{\pi\in\Pi} \sum_{t=1}^n \frac{\pi(a_t|x_t)}{\piref(a_t|x_t) + \beta} r_t,\\
\hat{\pi}_{\mathsf{LS}}&\defeq \arg\max_{\pi\in\Pi} \sum_{t=1}^n \ln(1+\beta \ips_t^\pi),\\
\hat{\pi}_{\mathsf{LS+freezing}}&\defeq \arg\max_{\pi\in\Pi} \sum_{t=1}^n \ln(1+\beta \ips_t^\pi)\onec{x \le \beta \ips_t^\pi}.
\end{align*}
In each case, $\beta>0$ is an hyperparameter.

\subsection{Additional Experiments for OP Learning and Selection}
\label{app:sec:add_exp}
In this section, 
we report the OP learning and selection results with two more different policies $\pi_{\text{good},\eps=0.01}$ and $\pi_{\text{bad},\eps=0.1}$ for completeness.
Note that the experimental results in the main text were with the policy $\pi_{\text{good},\eps=0.1}$.

Figure~\ref{fig:learning_additional} and Table~\ref{tab:selection_additional} summarize the OP learning and selection results, respectively. 
For $\pi_{\text{good},\eps=0.01}$, the behavior of the estimators aligns with that under $\pi_{\text{good},\eps=0.1}$ in Figure~\ref{fig:learning} and Table~\ref{tab:selection}. 
Specifically, \texttt{LS+freezing} consistently improves upon \texttt{LS} in learning, and \texttt{PUB} achieves the best or near-best performance in selection.

In contrast, for $\pi_{\text{bad},\eps=0.01}$, OP learning results show mixed behavior, likely due to the poor quality of the behavior policy. In OP selection, all methods perform poorly in the small-sample regime, failing to improve upon the \texttt{IW} baseline. We particularly note that, while \texttt{EB}, \texttt{LS}, and \texttt{PUB} show comparable performance, \texttt{WSWRKM} performs substantially worse in this erratic setting.

\begin{figure*}[bht]
\centering
\includegraphics[width=0.8\linewidth]{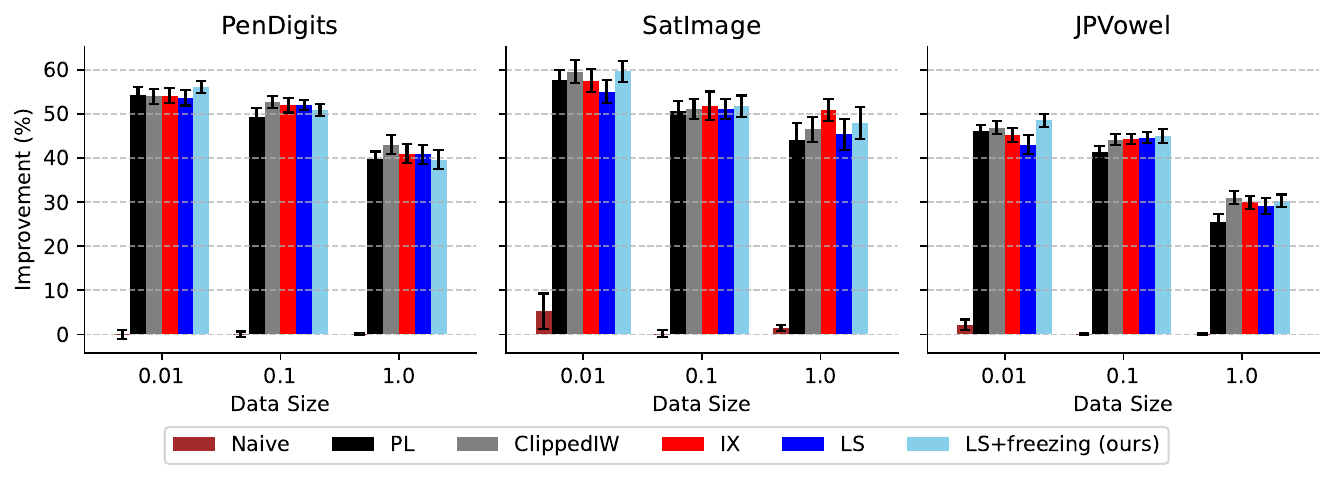}

{\small (a) Good policy + uniform policy with probability $\eps=0.01$ of choosing the good policy.}

\vspace{1em}
\includegraphics[width=0.8\linewidth]{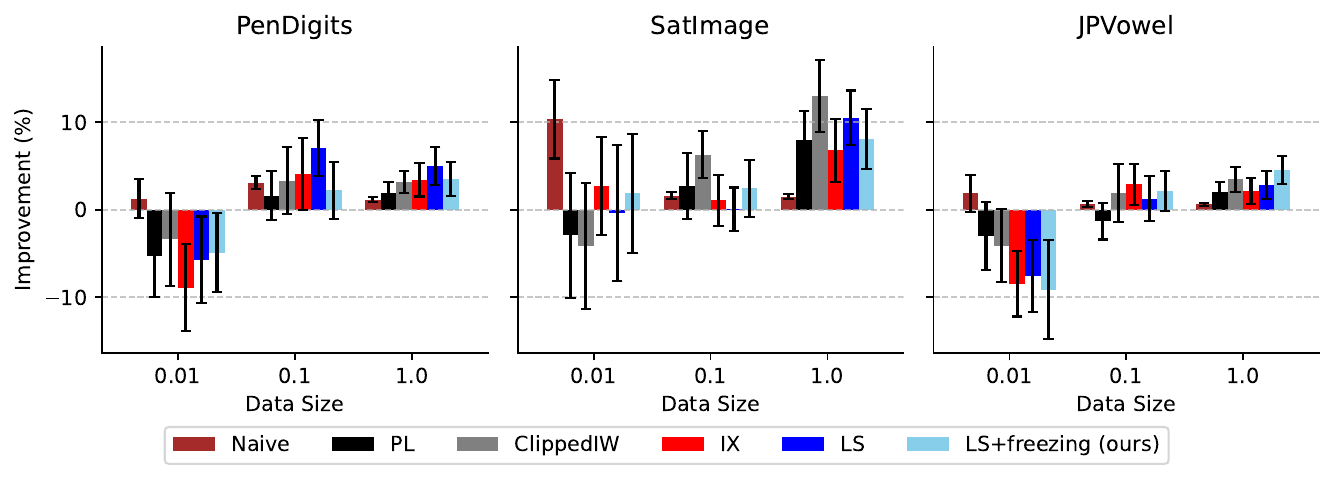}

{\small (b) Bad policy + uniform policy with probability $\eps=0.1$ of choosing the bad policy.}

\caption{Additional OP learning results. 
Compare to Figure~\ref{fig:learning} in the main text.
}
\label{fig:learning_additional}
\end{figure*}

\begin{table}[hbt]\vspace{-.5em}
  \centering
  \caption{Additional OP selection results. 
  Compare to Table~\ref{tab:selection} in the main text.
}
  {\small (a) Good policy + uniform policy with probability $\eps=0.01$ of choosing the good policy.}
  
  \vspace{0.5em}

  \scalebox{0.85}{
  \begin{tabular}{c|ccc|ccc|ccc}
    \toprule
    Dataset & \multicolumn{3}{c|}{PenDigits} & \multicolumn{3}{c|}{SatImage} & \multicolumn{3}{c}{JPVowel} \\
    \midrule
    Size & 0.01 & 0.1 & 1 & 0.01 & 0.1 & 1 & 0.01 & 0.1 & 1 \\
    \midrule
    \texttt{EB} 
    & \underline{50.90} & \underline{47.32} & \underline{40.24}
    & \textbf{52.34} & \underline{33.73} & \underline{33.98}
    & \underline{41.33} & \textbf{48.41} & \underline{32.09}
    \\
    \texttt{LS} 
    & 25.09 & 46.72 & \underline{39.30}
    & 29.40 & \underline{23.88} & \underline{35.94}
    & 18.38 & \underline{46.88} & \underline{32.15}
    \\
    \texttt{WSWRKM} 
    & \underline{35.22} & \underline{50.00} & 18.18
    & 44.01 & \textbf{35.09} & \textbf{39.00}
    & \textbf{45.76} & \underline{37.62} & \underline{31.17}
    \\
    \texttt{PUB} (ours)  
    & \textbf{51.02} & \textbf{51.35} & \textbf{41.33}
    & \underline{50.60} & \underline{32.05} & \underline{37.10}
    & \underline{35.26} & \underline{46.88} & \textbf{32.89}
    \\
    \bottomrule
  \end{tabular}
  }

  \vspace{1em}

  {\small (b) Bad policy + uniform policy with probability $\eps=0.1$ of choosing the bad policy.}

  \vspace{0.5em}

  \scalebox{0.85}{
  \begin{tabular}{c|ccc|ccc|ccc}
    \toprule
    Dataset & \multicolumn{3}{c|}{PenDigits} & \multicolumn{3}{c|}{SatImage} & \multicolumn{3}{c}{JPVowel} \\
    \midrule
    Size & 0.01 & 0.1 & 1 & 0.01 & 0.1 & 1 & 0.01 & 0.1 & 1 \\
    \midrule
    \texttt{EB} 
    & \textbf{-8.86} & \textbf{3.27} & \textbf{1.22}
    & \textbf{-8.36} & \textbf{0.93} & \textbf{2.68}
    & \textbf{-7.89} & \textbf{9.23} & \textbf{-0.79}
    \\
    \texttt{LS} 
    & \textbf{-8.86} & \textbf{3.27} & \textbf{1.22}
    & \textbf{-8.36} & \textbf{0.93} & \textbf{2.68}
    & \textbf{-7.89} & \underline{7.65} & \textbf{-0.79}
    \\
    \texttt{WSWRKM}
    & -15.79 & -0.56 & -1.50
    & -46.89 & -4.59 & 1.41 
    & -37.90 & -8.19 & -39.76
    \\
    \texttt{PUB} (ours)  
    & \textbf{-8.86} & \textbf{3.27} & \textbf{1.22}
    & \textbf{-8.36} & \textbf{0.93} & \textbf{2.68}
    & \underline{-8.66} & \underline{7.65} & \textbf{-0.79}
    \\
    \bottomrule
  \end{tabular}
  }
  \label{tab:selection_additional}
\end{table}

\end{document}

%% file: defs-kjun-v5-overleaf-colt.tex
\def\safedef#1{%
   \ifx#1\undefined
      \expandafter\def\expandafter#1%
   \else
      \errmessage{The \string#1 is defined already}%
      \expandafter\def\expandafter\tmp
   \fi
}

\usepackage{amsmath,bbm,amsfonts,amssymb}
\usepackage{dsfont} %
\usepackage{mathtools}

\usepackage{commath}
\usepackage{eqnarray}
\mathtoolsset{showonlyrefs=false}
\allowdisplaybreaks %

\usepackage{xspace}
\usepackage[T1]{fontenc}
\usepackage[utf8]{inputenc}

\usepackage{hyperref,url}
\usepackage{wrapfig}
\usepackage[export]{adjustbox}
\usepackage{graphicx}
\usepackage{multirow,hhline}%
\usepackage{booktabs}%
\usepackage{pbox} %
\usepackage{algorithm,algorithmic}

\usepackage[export]{adjustbox}

\usepackage{enumitem}
\usepackage[normalem]{ulem}

\newcommand{\guide}[1]{{\color{Violet}[#1]}}

\newcommand{\gray}[1]{{ \color[rgb]{.6,.6,.6} #1 }}

{\color{kjsaved}}%

\newcommand{\blue}[1]{{\color[rgb]{.3,.5,1}#1}}
\usepackage{framed}
\usepackage[most]{tcolorbox}
\definecolor{kjgray}{rgb}{.7,.7,.7}

\usepackage{mdframed}
\usepackage{lipsum}
\definecolor{kjgray}{rgb}{.7,.7,.7}
\makeatletter



\makeatletter
\renewcommand{\paragraph}{%
  \@startsection{paragraph}{4}%
  {\z@}{0.50ex \@plus 1ex \@minus .2ex}{-1em}%
  {\normalfont\normalsize\bfseries}%
}
\makeatother

\usepackage[customcolors,shade]{hf-tikz} %

\def\ddefloop#1{\ifx\ddefloop#1\else\ddef{#1}\expandafter\ddefloop\fi}

\def\ddef#1{\expandafter\def\csname #1#1\endcsname{\ensuremath{\mathbb{#1}}}}
\ddefloop ABCDFGHIJKLMNORSTUWXYZ\ddefloop 

\def\ddef#1{\expandafter\def\csname c#1\endcsname{\ensuremath{\mathcal{#1}}}}
\ddefloop ABCDEFGHIJKLMNOPQRSTUVWXYZ\ddefloop

\def\ddef#1{\expandafter\def\csname b#1\endcsname{\ensuremath{{\mathbf{#1}}}}}
\ddefloop ABCDEFGHIJKLMNOPQRSTUVWXYZ\ddefloop  
\def\ddef#1{\expandafter\def\csname b#1\endcsname{\ensuremath{{\boldsymbol{#1}}}}}
\ddefloop abcdeghijklmnopqrtsuvwxyz\ddefloop  %

\def\ddef#1{\expandafter\def\csname h#1\endcsname{\ensuremath{\hat{#1}}}}
\ddefloop ABCDEFGHIJKLMNOPQRSTUVWXYZabcdefghijklmnopqrsuvwxyz\ddefloop %
\def\ddef#1{\expandafter\def\csname hc#1\endcsname{\ensuremath{\hat{\mathcal{#1}}}}}
\ddefloop ABCDEFGHIJKLMNOPQRSTUVWXYZ\ddefloop
\def\ddef#1{\expandafter\def\csname hb#1\endcsname{\ensuremath{\hat{\mathbf{#1}}}}}
\ddefloop ABCDEFGHIJKLMNOPQRSTUVWXYZ\ddefloop %
\def\ddef#1{\expandafter\def\csname hb#1\endcsname{\ensuremath{\hat{\boldsymbol{#1}}}}}
\ddefloop abcdefghijklmnopqrstuvwxyz\ddefloop %

\def\ddef#1{\expandafter\def\csname t#1\endcsname{\ensuremath{\tilde{#1}}}}
\ddefloop ABCDEFGHIJKLMNOPQRSTUVWXYZabcdefgijklmnpqtsuvwxyz\ddefloop %
\def\ddef#1{\expandafter\def\csname tc#1\endcsname{\ensuremath{\tilde{\mathcal{#1}}}}}
\ddefloop ABCDEFGHIJKLMNOPQRSTUVWXYZ\ddefloop
\def\ddef#1{\expandafter\def\csname tb#1\endcsname{\ensuremath{\tilde{\mathbf{#1}}}}}
\ddefloop ABCDEFGHIJKLMNOPQRSTUVWXYZ\ddefloop
\def\ddef#1{\expandafter\def\csname tb#1\endcsname{\ensuremath{\tilde{\boldsymbol{#1}}}}}
\ddefloop abcdefghijklmnopqrstuvwxyz\ddefloop %

\def\ddef#1{\expandafter\def\csname bar#1\endcsname{\ensuremath{\bar{#1}}}}
\ddefloop ABCDEFGHIJKLMNOPQRSTUVWXYZabcdefghijklmnopqrtsuvwxyz\ddefloop
\def\ddef#1{\expandafter\def\csname barc#1\endcsname{\ensuremath{\bar{\mathcal{#1}}}}}
\ddefloop ABCDEFGHIJKLMNOPQRSTUVWXYZ\ddefloop
\def\ddef#1{\expandafter\def\csname barb#1\endcsname{\ensuremath{\bar{\mathbf{#1}}}}}
\ddefloop ABCDEFGHIJKLMNOPQRSTUVWXYZ\ddefloop
\def\ddef#1{\expandafter\def\csname barb#1\endcsname{\ensuremath{\bar{\boldsymbol{#1}}}}}
\ddefloop abcdefghijklmnopqrstuvwxyz\ddefloop %

\def\ddef#1{\expandafter\def\csname war#1\endcsname{\ensuremath{\overline{#1}}}}
\ddefloop ABCDEFGHIJKLMNOPQRSTUVWXYZabcdefghijklmnopqrtsuvwxyz\ddefloop
\def\ddef#1{\expandafter\def\csname warc#1\endcsname{\ensuremath{\overline{\mathcal{#1}}}}}
\ddefloop ABCDEFGHIJKLMNOPQRSTUVWXYZ\ddefloop
\def\ddef#1{\expandafter\def\csname warb#1\endcsname{\ensuremath{\overline{\mathbf{#1}}}}}
\ddefloop ABCDEFGHIJKLMNOPQRSTUVWXYZ\ddefloop
\def\ddef#1{\expandafter\def\csname warb#1\endcsname{\ensuremath{\overline{\boldsymbol{#1}}}}}
\ddefloop abcdefghijklmnopqrstuvwxyz\ddefloop %

\def\tilr{\tilde r}

\def\dt{\delta}

\def\eps{\varepsilon}
\def\epsilon{\varepsilon}
\def\th{\theta}

\def\Dt{\Delta}

\def\Th{\Theta}

\usepackage{pgffor}
\def\greeksymbols{alpha,beta,gamma,gam,delta,dt,eps,epsilon,zeta,eta,theta,th,iota,kappa,kap,lambda,lam,mu,nu,xi,pi,rho,sigma,sig,tau,phi,chi,psi,omega,om,Gamma,Gam,Delta,Dt,Theta,Th,Lambda,Lam,Pi,Sigma,Sig,Phi,Psi,Omega,Om}
\def\greeksymbolsnoeta{alpha,beta,gamma,gam,delta,dt,eps,epsilon,zeta,theta,th,iota,kappa,kap,lambda,lam,mu,nu,xi,pi,rho,sigma,sig,tau,phi,chi,psi,omega,om,Gamma,Gam,Delta,Dt,Theta,Th,Lambda,Lam,Pi,Sigma,Sig,Phi,Psi,Omega,Om} %

\foreach \x in \greeksymbolsnoeta{\expandafter\xdef\csname b\x\endcsname{\noexpand\ensuremath{\noexpand\boldsymbol{\csname \x\endcsname}}}}

\foreach \x in \greeksymbols{\expandafter\xdef\csname h\x\endcsname{\noexpand\ensuremath{\noexpand\hat{\csname \x\endcsname}}}}
\foreach \x in \greeksymbolsnoeta{\expandafter\xdef\csname hb\x\endcsname{\noexpand\ensuremath{\noexpand\hat{\noexpand\boldsymbol{ \csname \x\endcsname}}}}}

\foreach \x in \greeksymbols{\expandafter\xdef\csname bar\x\endcsname{\noexpand\ensuremath{\noexpand\bar{\csname \x\endcsname}}}}
\foreach \x in \greeksymbolsnoeta{%
\expandafter\xdef\csname barb\x\endcsname{\noexpand\ensuremath{\noexpand\bar{\noexpand\boldsymbol{ \csname \x\endcsname}}}}
}

\foreach \x in \greeksymbols{\expandafter\xdef\csname t\x\endcsname{\noexpand\ensuremath{\noexpand\tilde{\csname \x\endcsname}}}}
\foreach \x in \greeksymbolsnoeta{\expandafter\xdef\csname tb\x\endcsname{\noexpand\ensuremath{\noexpand\tilde{\noexpand\boldsymbol{ \csname \x\endcsname}}}}}

\providecommand{\normz}[2][-1]{
\ensuremath{\mathinner{
\ifthenelse{\equal{#1}{-1}}{ %
\!\left\|#2\right\|}{}
\ifthenelse{\equal{#1}{0}}{ %
\|#2\|}{}
\ifthenelse{\equal{#1}{1}}{ %
\bigl\|#2\bigr\|}{}
\ifthenelse{\equal{#1}{2}}{ %
\Bigl\|#2\Bigr\|}{}
\ifthenelse{\equal{#1}{3}}{ %
\biggl\|#2\biggr\|}{}
\ifthenelse{\equal{#1}{4}}{ %
\Biggl\|#2\Biggr\|}{}
}} %
}  %
\providecommand{\floor}[2][-1]{
\ensuremath{\mathinner{
\ifthenelse{\equal{#1}{-1}}{ %
\!\left\lfloor#2\right\rfloor}{}
\ifthenelse{\equal{#1}{0}}{ %
\lfloor#2\rfloor}{}
\ifthenelse{\equal{#1}{1}}{ %
\!\bigl\lfloor#2\bigr\rfloor}{}
\ifthenelse{\equal{#1}{2}}{ %
\!\Bigl\lfloor#2\Bigr\rfloor}{}
\ifthenelse{\equal{#1}{3}}{ %
\!\biggl\lfloor#2\biggr\rfloor}{}
\ifthenelse{\equal{#1}{4}}{ %
\!\Biggl\lfloor#2\Biggr\rfloor}{}
}} %
}
\providecommand{\ceil}[2][-1]{
\ensuremath{\mathinner{
\ifthenelse{\equal{#1}{-1}}{ %
\!\left\lceil#2\right\rceil}{}
\ifthenelse{\equal{#1}{0}}{ %
\lceil#2\rceil}{}
\ifthenelse{\equal{#1}{1}}{ %
\!\bigl\lceil#2\bigr\rceil}{}
\ifthenelse{\equal{#1}{2}}{ %
\!\Bigl\lceil#2\Bigr\rceil}{}
\ifthenelse{\equal{#1}{3}}{ %
\!\biggl\lceil#2\biggr\rceil}{}
\ifthenelse{\equal{#1}{4}}{ %
\!\Biggl\lceil#2\Biggr\rceil}{}
}} %
}
\newcommand{\fr}[2]{ { \frac{#1}{#2} }}

\newcommand{\til}[1]{{\ensuremath{\tilde{#1}}}}

\def\wed{\wedge}
\def\tsty{\textstyle}

\def\cd{\cdot}

\def\rarrow{\ensuremath{\rightarrow}}

\definecolor{mygrn}{rgb}{0,.8,0}
\definecolor{myred}{rgb}{.8,0,0}

\DeclareMathOperator{\EE}{\mathbb{E}} %
\DeclareMathOperator{\VV}{\mathbb{V}}
\DeclareMathOperator{\PP}{\mathbb{P}}

\DeclareMathOperator{\Var}{{\mathrm{Var}}}

\DeclareMathOperator{\one}{\mathds{1}\hspace{-.1em}}

\providecommand{\onec}[2][-1]{
\ensuremath{\mathinner{
\one\cbr{#2}
} %
}  %
}

\DeclarePairedDelimiterX{\inp}[2]{\langle}{\rangle}{#1, #2}

\newcommand\declareop[3]{%
  \newcommand#1{%
    \mskip\muexpr\medmuskip*#2\relax
    {#3}%
    \mskip\muexpr\medmuskip*#2\relax
}}
\declareop\capprox{1}{{\sr{\const}{\approx}}} %
\declareop\logapprox{1}{{\sr{\mathsf{log}}{\approx}}} %

\newcommand{\gsim}{\mathop{}\!\gtrsim}

\def\const{\mathsf{const}}

\def\suchthat{\ensuremath{\text{ s.t. }}}

\usepackage{pifont}%
\newcommand{\sr}{\stackrel}

\makeatletter
\newcommand{\vast}{\bBigg@{3}}
\newcommand{\Vast}{\bBigg@{4}}
\makeatother

\newenvironment{talign*}
 {\let\displaystyle\textstyle\csname align*\endcsname}
 {\endalign}

\def\chrulefill{\leavevmode\leaders\hrule height 0.7ex depth \dimexpr0.4pt-0.7ex\hfill\kern0pt}

%% file: main-colt.bbl
\newcommand{\noopsort}[1]{}
\begin{thebibliography}{32}
\providecommand{\natexlab}[1]{#1}
\providecommand{\url}[1]{\texttt{#1}}
\expandafter\ifx\csname urlstyle\endcsname\relax
  \providecommand{\doi}[1]{doi: #1}\else
  \providecommand{\doi}{doi: \begingroup \urlstyle{rm}\Url}\fi

\bibitem[Aouali et~al.(2023)Aouali, Brunel, Rohde, and Korba]{aouali23exponential}
Imad Aouali, Victor-Emmanuel Brunel, David Rohde, and Anna Korba.
\newblock Exponential smoothing for off-policy learning.
\newblock In \emph{Proc. Int. Conf. Mach. Learn.}, pages 984--1017, 2023.

\bibitem[Cover(1991)]{Cover1991}
Thomas~M Cover.
\newblock Universal portfolios.
\newblock \emph{Math. Financ.}, 1\penalty0 (1):\penalty0 1--29, 1991.

\bibitem[Cover and Ordentlich(1996)]{Cover--Ordentlich1996}
Thomas~M Cover and Erik Ordentlich.
\newblock Universal portfolios with side information.
\newblock \emph{{IEEE} Trans. Inf. Theory}, 42\penalty0 (2):\penalty0 348--363, 1996.

\bibitem[Cover and Thomas(2006)]{Cover--Thomas2006}
Thomas~M Cover and Joy~A. Thomas.
\newblock \emph{Elements of information theory}.
\newblock John Wiley \& Sons, 2006.

\bibitem[Gabbianelli et~al.(2024)Gabbianelli, Neu, and Papini]{gabbianelli24importance}
Germano Gabbianelli, Gergely Neu, and Matteo Papini.
\newblock Importance-weighted offline learning done right.
\newblock In \emph{Proc. Int. Conf. Algorithmic Learn. Theory}, pages 614--634. PMLR, 2024.

\bibitem[Horvitz and Thompson(1952)]{horvitz52generalization}
Daniel~G Horvitz and Donovan~J Thompson.
\newblock {A generalization of sampling without replacement from a finite universe}.
\newblock \emph{J. Am. Statist. Assoc.}, 47\penalty0 (260):\penalty0 663--685, 1952.

\bibitem[Jin et~al.(2022)Jin, Ren, Yang, and Wang]{jin22policy}
Ying Jin, Zhimei Ren, Zhuoran Yang, and Zhaoran Wang.
\newblock Policy learning ``without'' overlap: Pessimism and generalized empirical {B}ernstein's inequality.
\newblock \emph{arXiv preprint arXiv:2212.09900}, 2022.

\bibitem[Karampatziakis et~al.(2020)Karampatziakis, Langford, and Mineiro]{karampatziakis20empirical}
Nikos Karampatziakis, John Langford, and Paul Mineiro.
\newblock Empirical likelihood for contextual bandits.
\newblock \emph{Adv. Neural Inf. Proc. Syst.}, 33:\penalty0 9597--9607, 2020.

\bibitem[Karampatziakis et~al.(2021)Karampatziakis, Mineiro, and Ramdas]{karampatziakis21off}
Nikos Karampatziakis, Paul Mineiro, and Aaditya Ramdas.
\newblock Off-policy confidence sequences.
\newblock In \emph{Proc. Int. Conf. Mach. Learn.}, pages 5301--5310, 2021.

\bibitem[Kuzborskij and Szepesvári(2019)]{kuzborskij19efron}
Ilja Kuzborskij and Csaba Szepesvári.
\newblock Efron-{S}tein {PAC}-{B}ayesian {I}nequalities.
\newblock arXiv:1909.01931, 2019.

\bibitem[Kuzborskij et~al.(2021)Kuzborskij, Vernade, Gyorgy, and Szepesvari]{kuzborskij21confident}
Ilja Kuzborskij, Claire Vernade, Andras Gyorgy, and Csaba Szepesvari.
\newblock Confident off-policy evaluation and selection through self-normalized importance weighting.
\newblock In \emph{Proc. Int. Conf. Artif. Int. Statist.}, 2021.

\bibitem[Langford and Zhang(2007)]{langford07epoch}
John Langford and Tong Zhang.
\newblock The epoch-greedy algorithm for multi-armed bandits with side information.
\newblock In \emph{Adv. Neural Inf. Proc. Syst.}, 2007.

\bibitem[Li et~al.(2010)Li, Chu, Langford, and Schapire]{li10acontextual}
Lihong Li, Wei Chu, John Langford, and Robert~E. Schapire.
\newblock {A Contextual-Bandit Approach to Personalized News Article Recommendation}.
\newblock In \emph{Proc. Int. Conf. World Wide Web}, pages 661--670, 2010.

\bibitem[Liu et~al.(2020)Liu, Mania, and Jordan]{liu19competing}
Lydia~T Liu, Horia Mania, and Michael~I Jordan.
\newblock {Competing Bandits in Matching Markets}.
\newblock In \emph{Proc. Int. Conf. Artif. Int. Statist.}, volume 108, 2020.

\bibitem[London and Sandler(2019)]{london19bayesian}
Ben London and Ted Sandler.
\newblock Bayesian counterfactual risk minimization.
\newblock In \emph{Proc. Int. Conf. Mach. Learn.}, pages 4125--4133, 2019.

\bibitem[Maurer and Pontil(2009)]{maurer09empirical}
Andreas Maurer and Massimiliano Pontil.
\newblock Empirical {B}ernstein bounds and sample variance penalization.
\newblock In \emph{Conf. Learn. Theory}, 2009.

\bibitem[Orabona and Jun(2024)]{orabona24tight}
Francesco Orabona and Kwang-Sung Jun.
\newblock Tight concentrations and confidence sequences from the regret of universal portfolio.
\newblock \emph{{IEEE} Trans. Inf. Theory}, 70\penalty0 (1):\penalty0 436--455, 2024.

\bibitem[Robins and Rotnitzky(1995)]{robins95semiparametric}
James~M Robins and Andrea Rotnitzky.
\newblock {Semiparametric efficiency in multivariate regression models with missing data}.
\newblock \emph{J. Am. Statist. Assoc.}, 90\penalty0 (429):\penalty0 122--129, 1995.

\bibitem[Ryu and Bhatt(2024)]{Ryu--Bhatt2024}
J.~Jon Ryu and Alankrita Bhatt.
\newblock On confidence sequences for bounded random processes via universal gambling strategies.
\newblock \emph{{IEEE} Trans. Inf. Theory}, 70\penalty0 (10):\penalty0 7143--7161, 2024.

\bibitem[Ryu and Wornell(2024)]{Ryu--Wornell2024}
Jongha~Jon Ryu and Gregory~W. Wornell.
\newblock Gambling-based confidence sequences for bounded random vectors.
\newblock In \emph{Proc. Int. Conf. Mach. Learn.}, volume 235, pages 42856--42869. PMLR, 21--27 Jul 2024.

\bibitem[Sakhi et~al.(2023)Sakhi, Alquier, and Chopin]{sakhi23pac}
Otmane Sakhi, Pierre Alquier, and Nicolas Chopin.
\newblock Pac-bayesian offline contextual bandits with guarantees.
\newblock In \emph{Proc. Int. Conf. Mach. Learn.}, pages 29777--29799, 2023.

\bibitem[Sakhi et~al.(2024)Sakhi, Aouali, Alquier, and Chopin]{sakhi24logarithmic}
Otmane Sakhi, Imad Aouali, Pierre Alquier, and Nicolas Chopin.
\newblock Logarithmic smoothing for pessimistic off-policy evaluation, selection and learning.
\newblock In \emph{Adv. Neural Inf. Proc. Syst.}, 2024.

\bibitem[Schwartz et~al.(2017)Schwartz, Bradlow, and Fader]{schwartz2017customer}
Eric~M Schwartz, Eric~T Bradlow, and Peter~S Fader.
\newblock Customer acquisition via display advertising using multi-armed bandit experiments.
\newblock \emph{Marketing Science}, 36\penalty0 (4):\penalty0 500--522, 2017.

\bibitem[Swaminathan and Joachims(2015)]{swaminathan15batch}
Adith Swaminathan and Thorsten Joachims.
\newblock Batch learning from logged bandit feedback through counterfactual risk minimization.
\newblock \emph{J. Mach. Learn. Res.}, 16\penalty0 (52):\penalty0 1731--1755, 2015.

\bibitem[Vanschoren et~al.(2013)Vanschoren, van Rijn, Bischl, and Torgo]{OpenML2013}
Joaquin Vanschoren, Jan~N. van Rijn, Bernd Bischl, and Luis Torgo.
\newblock {OpenML}: networked science in machine learning.
\newblock \emph{SIGKDD Explorations}, 15\penalty0 (2):\penalty0 49--60, 2013.

\bibitem[Ville(1939)]{ville39etude}
Jean Ville.
\newblock {Etude critique de la notion de collectif}.
\newblock \emph{Bull. Amer. Math. Soc}, 45\penalty0 (11):\penalty0 824, 1939.

\bibitem[Wang and Ramdas(2023)]{wang2023catoni}
Hongjian Wang and Aaditya Ramdas.
\newblock Catoni-style confidence sequences for heavy-tailed mean estimation.
\newblock \emph{Stoch. Process. Their Appl.}, 163:\penalty0 168--202, 2023.

\bibitem[Wang et~al.(2024)Wang, Krishnamurthy, and Slivkins]{wang24oracle}
Lequn Wang, Akshay Krishnamurthy, and Aleksandrs Slivkins.
\newblock Oracle-efficient pessimism: Offline policy optimization in contextual bandits.
\newblock In \emph{Proc. Int. Conf. Artif. Int. Statist.}, 2024.

\bibitem[Waudby-Smith and Ramdas(2020)]{Waudby-Smith--Ramdas2020a}
Ian Waudby-Smith and Aaditya Ramdas.
\newblock Confidence sequences for sampling without replacement.
\newblock In \emph{Adv. Neural Inf. Proc. Syst.}, volume~33, pages 20204--20214. Curran Associates, Inc., 2020.

\bibitem[Waudby-Smith and Ramdas(2024)]{Waudby-Smith--Ramdas2020b}
Ian Waudby-Smith and Aaditya Ramdas.
\newblock Estimating means of bounded random variables by betting.
\newblock \emph{J. R. Stat. Soc. B}, 86\penalty0 (1):\penalty0 1--27, 2024.

\bibitem[Waudby-Smith et~al.(2022)Waudby-Smith, Wu, Ramdas, Karampatziakis, and Mineiro]{Waudby-Smith--Wu--Ramdas--Karampatziakis--Mineiro2022}
Ian Waudby-Smith, Lili Wu, Aaditya Ramdas, Nikos Karampatziakis, and Paul Mineiro.
\newblock Anytime-valid off-policy inference for contextual bandits.
\newblock \emph{arXiv preprint arXiv:2210.10768}, 2022.

\bibitem[Zenati et~al.(2023)Zenati, Diemert, Martin, Mairal, and Gaillard]{zenati23sequential}
Houssam Zenati, Eustache Diemert, Matthieu Martin, Julien Mairal, and Pierre Gaillard.
\newblock Sequential counterfactual risk minimization.
\newblock In \emph{Proc. Int. Conf. Mach. Learn.}, pages 40681--40706, 2023.

\end{thebibliography}
